\documentclass[preprint]{imsart}

\usepackage[pdftex]{graphicx}
\graphicspath{
  {figures/various/}
  {figures/natural2/}
  {figures/additional/}
  {figures/noise/}
}

\usepackage[ruled,vlined]{algorithm2e}

\usepackage{url}
\usepackage{array}
\usepackage[caption=false,font=normalsize,labelfont=sf,textfont=sf]{subfig}
\usepackage{fixltx2e}
\usepackage{rotating}
\usepackage{tabularx}
\usepackage{multirow}

\usepackage{tepper}

\begin{document}

\begin{frontmatter}
\title{On the Role of Contrast and Regularity in Perceptual Boundary Saliency}
\runtitle{Contrast and Regularity in Perceptual Boundary Saliency}

\begin{aug}
  \author{
    \fnms{Mariano} \snm{Tepper}%
    \corref{}%
    \thanksref{t2}
    \ead[label=e1]{mariano.tepper@duke.edu}
  }
  
  \thankstext{t2}{Work done while MT was with the Departamento de Computaci\'on, Facultad de Ciencias Exactas y Naturales, Universidad de Buenos Aires, Argentina.}

  \address{
    Department of Electrical and Computer Engineering, Duke University, USA.\\
    \printead{e1}
  }

  \author{
    \fnms{Pablo} \snm{Mus\'e}
    \ead[label=e3]{pmuse@fing.edu.uy}%
  }

  \address{
    Instituto de Ingenier\'ia El\'ectrica, Facultad de Ingenier\'ia, Universidad de la Rep\'ublica, Uruguay.\\
    \printead{e3}
  }

  \author{
    \fnms{Andr\'es} \snm{Almansa}
    \ead[label=e4]{andres.almansa@telecom-paristech.fr}%
  }

  \address{
    CNRS - LTCI UMR5141, Telecom ParisTech, France.\\
    \printead{e4}
  }

  \runauthor{M. Tepper et al.}

\end{aug}


\begin{abstract}
Mathematical Morphology proposes to extract shapes from images as connected components of level sets. These methods prove very suitable for shape recognition and analysis.
We present a method to select the perceptually significant (i.e., contrasted) level lines (boundaries of level sets), using the Helmholtz principle as first proposed by Desolneux~\etal Contrarily to the classical formulation by Desolneux~\etal where level lines must be entirely salient, the proposed method allows to detect partially salient level lines, thus resulting in more robust and more stable detections. We then tackle the problem of combining two gestalts as a measure of saliency and propose a method that reinforces detections. Results in natural images show the good performance of the proposed methods.
\end{abstract}

\begin{keyword}
\kwd{topographic maps}
\kwd{level lines}
\kwd{edge detection}
\kwd{Helmholtz principle}
\end{keyword}

\end{frontmatter}

\section{Introduction}

Shape plays a key role in our cognitive system: in the perception of shape lies the beginning of concept formation.

Artists have implicitly acknowledged the importance of shapes since the dawn of times. Indeed, despite that lines do not divide objects from their background in the real world, line drawings are present in much of our earliest recorded art and, remarkably, remained unchanged through history, see Figure~\ref{fig:lineDrawing}.

\begin{figure}

  \centerline{
    \includegraphics[width=.48\columnwidth]{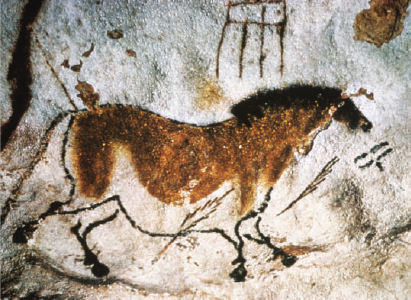} 
    \includegraphics[width=.48\columnwidth]{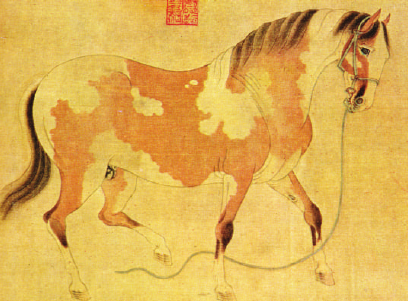}
  }

  \caption{Lines are used to convey the outer contours of the horses in a very similar way in these drawings, one from 15,000 BC (left: Chinese Horse, paleolithic cave painting at Lascaux, France) and the other from AD 1300 (right: Jen Jen-fa, detail from The Lean Horse and the Fat Horse, Peking Museum, China). Reprinted by permission from Macmillan Ltd: NATURE~\cite{cavanagh05}, copyright 2005.}
  \label{fig:lineDrawing}
\end{figure}

Although art may provide clues to understand shape perception, it tells us little from the formal point of view. Let us begin by defining what is a shape.

Phenomenologists~\cite{attneave54} conceive shape as a subset of an image, digital or perceptual, endowed with some qualities permitting its recognition. In this sense, both concepts, shape and recognition, are intrinsically intertwined: one has to define what is a shape in such a way that its recognition can be performed.

Following these lines of thought, gestaltists~\cite{arnheim} regard shape perception as the grasping of structural features found in or imposed upon the stimulus material. The Gestalt school has extensively studied phenomena that unveil and justify this definition~\cite{kanizsa79,wertheimer38}.

Formally, shapes can be defined by extracting contours from solid objects. In this context, shapes are represented and analyzed from an infinite-dimensional approach in which a shape is the locus of an infinite number of points~\cite{krim06}. This point of view leads to the active contours formulation~\cite{kass88} or to level-sets methods~\cite{serra83}. Although these shapes can be defined in any number of dimensions, e.g. the contour of a three dimensional solid object is a surface, we will restrict ourselves to the two dimensional case, following~Lisani \etal~\cite{lisani03-shape} and Cao \etal~\cite{cao08theory}.

We define an image as a function $u: \R^2 \rightarrow \R$, where $u(x)$ represents the gray level or luminance at point $x$. Our first task is to extract the topological information of an image, independent of the unknown contrast change function of the acquisition system.  This contrast change function can be modeled as a continuous and increasing function $g$. The observed data of an image $u$ might be any such $g(u)$. This simple argument leads to select the level sets~\cite{serra83}, or level lines, as a complete and contrast-invariant image description~\cite{caselles99,caselles10}.

Given an image $u$, the upper level set $\mathcal{X}_{\lambda}$ and the lower level set $\mathcal{X}^{\lambda}$ of level $\lambda$ are subsets of $\R^2$ defined by~\cite{caselles10}
\begin{align}
  \mathcal{X}_{\lambda} &= \lbrace x \in \R^2 \ |\ u(x) \geq \lambda \rbrace \textbf{,} \\
  \mathcal{X}^{\lambda} &= \lbrace x \in \R^2 \ |\ u(x) < \lambda \rbrace \textbf{.}
\end{align}
If the image $u$
is lower (\emph{resp.} upper) semi-continuous,
it can be reconstructed from the collection of its upper (\emph{resp.} lower) level sets by using the superposition principle~\cite{matheron75}:
\begin{align}
  u(x) &= \sup  \lbrace \lambda\ |\ x \in \mathcal{X}_{\lambda} \rbrace \textbf{,} \\
  u(x) &= \inf  \lbrace \lambda\ |\ x \in \mathcal{X}^{\lambda} \rbrace \textbf{.}
\end{align}
We define the boundaries of the connected components of a level set as a level line.

A gray-level digital image $u_d$ is a discrete function in a rectangular grid that takes values in a finite set, typically integer values between 0 and 255. To obtain a grid independent representation, we can consider an interpolation $u$ of $u_d$ with the desired degree of regularity (i.e., $u$ can be $C^1$, $C^2$, etc.). In this work we use bilinear interpolation, in which case the level lines have the following properties:
\begin{itemize}
  \item for almost all $\lambda$, the level lines are closed Jordan curves;
  \item by topological inclusion, level lines form a partially ordered set.
\end{itemize}
For extracting the level lines of such a bilinearly interpolated image we make use of the Fast Level Set Transform (FLST)~\cite{monasse00}. Notice that the FLST correctly handles singularities such as saddle points. We call this collection of level lines (along with their level) a topographic map.

In general, the topographic map is an infinite set and so only quantized grey levels are considered, ensuring that the set is finite. Since the connected components of level sets are ordered by the inclusion relation, the topographic map may be embedded in a hierarchical representation. To make things simple, a level line $L_i$ is a descendant of another line $L_j$ in the hierarchy if and only if $L_i$ is included in the interior of $L_j$. Figure~\ref{fig:topographicMap} depicts a simple example.

\begin{figure}
    \centerline{
        \hfill
        \includegraphics[height=.3\columnwidth]{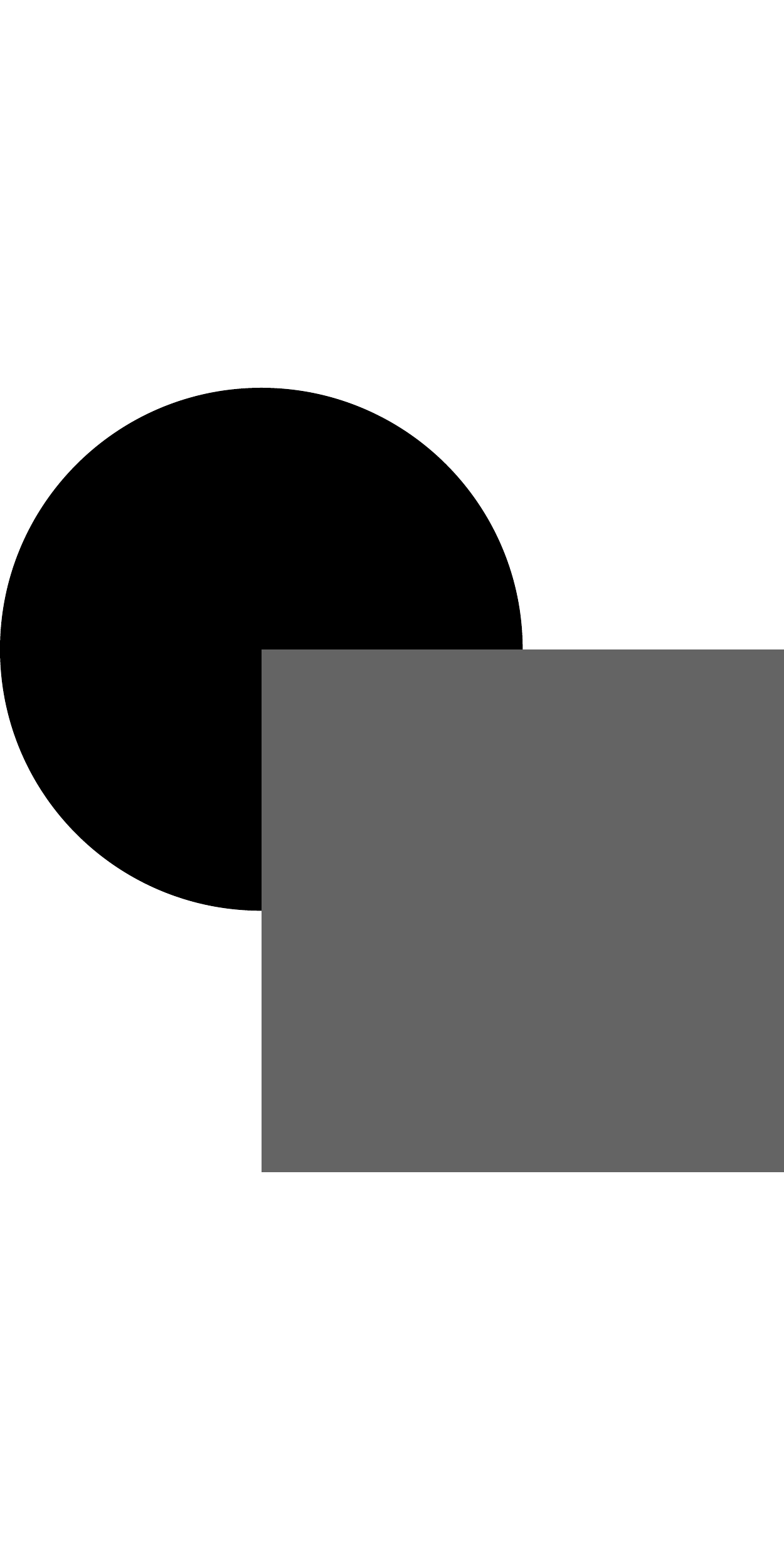}
        \hfill
        \includegraphics[height=.3\columnwidth]{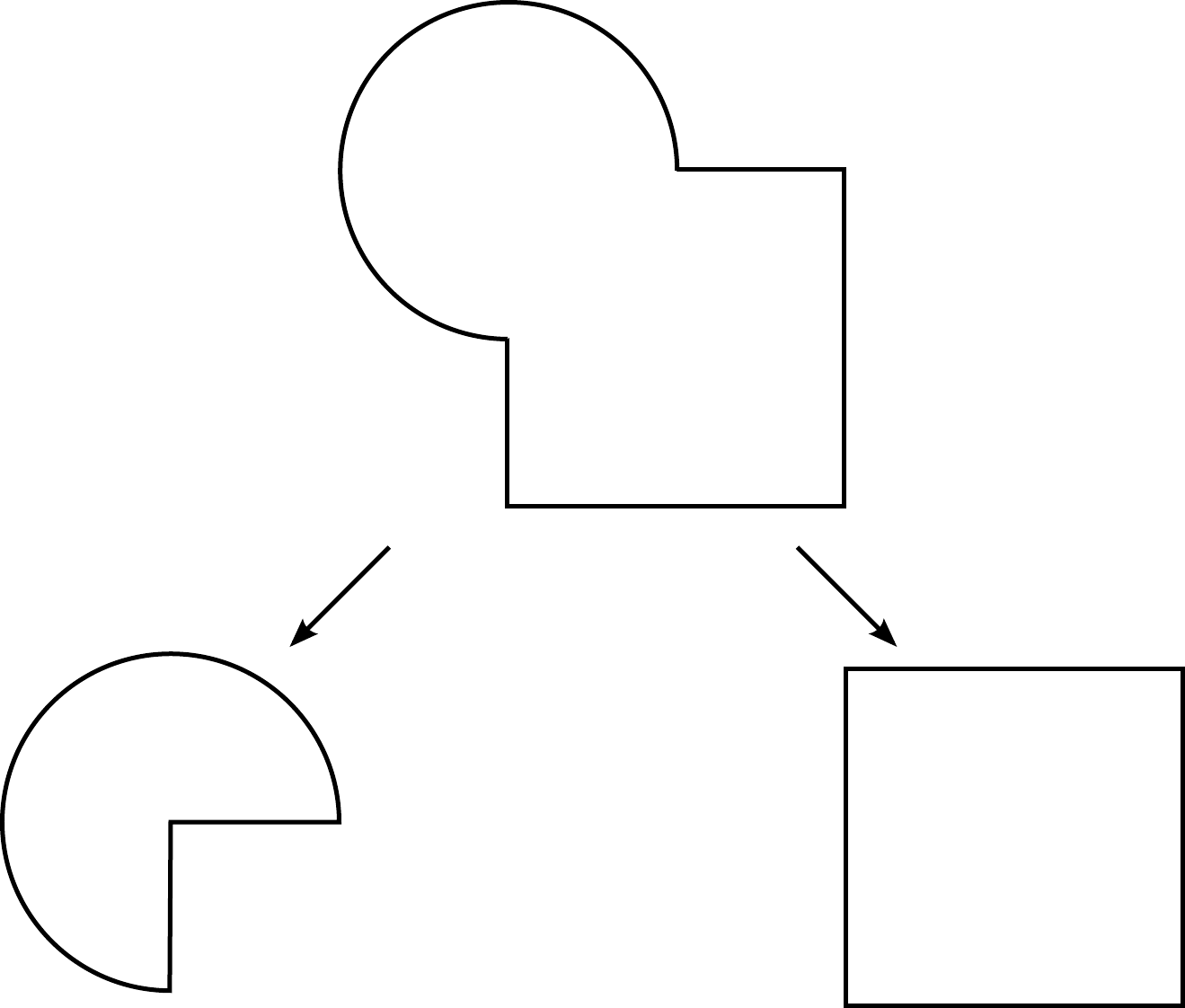}
        \hfill
    }
    
    \caption{On the left, original image. On the right, the hierarchical representation of the topographic map.}
    \label{fig:topographicMap}
\end{figure}

The Mathematical Morphology school~\cite{matheron75,serra83} has extensively studied the topographic map and its level sets, producing a whole set of tools for image analysis. Smoothing filters, usually described by Partial Differential Equations (PDE), can be proven to have an equivalent formulation in terms of iterated morphological operators~\cite{morelPDEs}. Hence, edge detectors can then be directly expressed by combining these operators.

The previous requirement leads us to define the set of level lines as a complete and contrast invariant image representation. In apparent contradiction to this fact, many authors, like Attneave, argue that ``information is concentrated along contours (regions where contrast changes abruptly)''~\cite{attneave54}. For example, edge detectors, from which the most renowned is Canny's~\cite{canny86}, rely on this fact. In summary, only a subset of the topographic map is necessary to obtain a \emph{perceptually} complete description.

The search for perceptually important lines will focus on unexpected configurations, rising from the perceptual laws of Gestalt Theory~\cite{kanizsa79,wertheimer38}.
From an algorithmic point of view, the main problem with Gestalt rules is their qualitative nature.
Desolneux \etal~\cite{desolneux08} developed a detection theory which seeks to provide a quantitative assessment of gestalts. This theory is often referred as Computational Gestalt and it has been successfully applied to numerous gestalts and detection problems~\cite{cao2005,grompone10,rabin09}. It is primarily based on the Helmholtz principle which states that conspicuous structures may be viewed as exceptions to randomness. In this approach, there is no need to characterize the elements one wishes to detect but contrarily, the elements
one wishes to avoid detecting, i.e., the background model. When an element sufficiently deviates from the background model, it is considered meaningful and thus, detected.

Within this framework, Desolneux~\etal~\cite{dmm01} proposed an algorithm to detect contrasted level lines in grey level images, called meaningful boundaries. Further improvements to this algorithm were proposed by Cao \etal~\cite{cao2005}.

In this work, we build upon these methods, presenting several contributions:
\paragraph{\textbf{From global to partial curve saliency}.} The original meaningful boundaries are totally salient curves (i.e., every point in the curve is salient). We propose a modification that allows detecting partially salient curves as meaningful boundaries. This definition agrees more tightly to the observation that pieces of level lines correspond to object contours and also yields more robust results.
\paragraph{\textbf{An extended definition of saliency}.} The criterion used to establish saliency in the original meaningful boundaries algorithm is contrast. Cao \etal~\cite{cao2005} proposed to determine saliency as a cooperation of two criteria: contrast and regularity. We study some theoretical and practical issues in their formulation. We then present a new formulation in which both aforementioned criteria compete, instead of cooperating. It is theoretically sound and yields improved detections, with respect to the ones obtained by using only contrast. The previous partial curve saliency criterion proves determinant in this new formulation

Strictly speaking, all the proposed algorithms are only invariant to affine contrast changes. This can be easily proven when contrast (i.e., the gradient magnitude) is used as the saliency measure~\cite[Lemma 1, p.~19]{cao08theory}. Nevertheless, the set of meaningful boundaries is not significantly affected by slight deviations from this class of contrast changes.

As a side note, we point out that there are two remaining steps to address in order to develop a complete shape detection system: smoothing, and geometrical invariance. Let us briefly discuss them for the sake of completeness.

First, during the acquisition, details much too fine to be perceptually relevant are introduced. It is necessary to use a suitable filtering mechanism. Invariance to these fine details may be handled by an appropriate smoothing procedure, i.e., the Affine morphological Scale Space (AMSS)~\cite{moisan98} or by a subsequent suitable shape description method~\cite{tepper09matching}.

Second, representations must be invariant to weak projective transformations. It can be shown that all planar curves within a large class can be mapped arbitrarily close to a circle by projective transformations~\cite{astrom95-limitations}. Moreover, full projective invariance is neither perceptually real (humans have great difficulties to recognize objects under strong perspective effects) nor computationally tractable. In this sense, affine invariance is the most we can impose in practice. At the same time, the effect of any optical acquisition system can be modeled by a convolution with a smoothing radial kernel. It does not commute with projective transformations and must be taken into account in the recognition process. A multiscale analysis is the only feasible way to treat it correctly. Both concepts, affine invariance and multiscale analysis are consistently integrated in the work by Morel and Yu~\cite{morel09ASIFT}.

The aforementioned tools that cover these issues can be directly applied to the level lines detected by our method. For a wide perspective of the complete shape recognition chain see the book by Cao~\etal~\cite{cao08theory}.

The paper is structured as follows.
In Section~\ref{sec:meaningfulBoundaries} we recall the definition of meaningful
boundaries and present a generalization that allows to detect partially salient curves.
In Section~\ref{sec:meaningfulSmoothBoundaries} we address the combination of contrast and regularity for the detection of meaningful boundaries.
We conclude in Section~\ref{sec:conclusions}.

\section{Meaningful Contrasted Boundaries}
\label{sec:meaningfulBoundaries}

Let us begin by formally explaining the meaningful boundaries algorithm by Desolneux \etal~\cite{dmm01}.

Let $C$ be a continuous level line of the (bilinearly interpolated) image $u$.
We consider a discrete sampling of this curve, and denote it by $x_0, x_1, \dots, x_{n-1}$
\footnote{This corresponds to the following 2 steps: i) The intersection of the continuous level-line $C$ with the Qedgels of the image gives a set of $m$ points as explained in \cite{caselles10}. ii) We sample $n=\lfloor m/2 \rfloor$ points by taking one out of every two points}.
This particular sampling is chosen to ensure that $|Du|(x_i)$ and $|Du|(x_{i+1})$ are statistically independent almost everywhere when pixel values of $u$ are considered to be independent  The gradient magnitude is computed using a standard finite difference scheme on a $2 \times 2$ neighborhood.

\begin{notation}
Let $H_c$ be the tail histogram of $|Du|$, defined by
\begin{equation}
H_c (\mu) \stackrel{\mathrm{def}}{=} \frac{\# \{ x \in u,\ |Du|(x) > \mu \}}{\# \{ x \in u,\ |Du|(x) > \min_{x \in u} |Du|(x) \}},
\end{equation}
where $Du$ can be computed by a standard finite differences scheme on a $2 \times 2$ neighborhood.
\label{not:H_c}
\end{notation}
\begin{definition}
\label{def:nfaContrastedCurve}
\textnormal{(Desolneux~\etal~\cite{dmm01})}
Let $\mathcal{C}$ be a finite set of $N_{ll}$ level lines of $u$. A level line $C \in \mathcal{C}$ is a DMM \meps-meaningful contrasted boundary (DMM-MCB) if
\begin{equation}
\NFA(C) \stackrel{\mathrm{def}}{=} N_{ll} \ H_c ( \min_{x \in C} |Du|(x) ) ^{l/2} < \eps
\end{equation}
where $l$ is the length of $C$. This number is called number of false alarms (NFA) of $C$.
\end{definition}
Actually, $l$ denotes the Euclidean length of the discrete approximation of $C$. In \cite{cao08theory} the authors assume that $l=2n$, but we found that this approximation is not accurate enough, which leads us to make here the distinction between $l$ and $2n$.

Algorithm~\ref{algo:meaningfulBoundaries} shows a possible procedure to obtain all \meps-meaningful contrasted boundaries.
\begin{algorithm}[t]
\SetKwInOut{Input}{input}\SetKwInOut{Output}{output}

\Input{An image $u$ and a scalar \meps.}
\Output{A set of closed curves $\mathcal{S}_\mathrm{res}$.}

$\mathcal{S} \gets \mathrm{FLST}(u)$\tcp*{Compute the set of level lines}
$N_{ll} \gets \#\{ \mathcal{S} \}$\;
Compute the tail histogram $H_c$ of $|Du|$\;
$\mathcal{S}_\mathrm{res} \gets \emptyset$\;
\For{$C \in \mathcal{S}$}{
Compute the length $l$ of $C$\;
$\displaystyle \mu \gets \min_{x \in C} |Du|(x)$\;
$\displaystyle \mathrm{nfa}_C \leftarrow N_{ll} \ H_c ( \mu ) ^{l/2}$\;
\lIf{$\mathrm{nfa}_C < \eps$}{
$\mathcal{S_\mathrm{res}} \gets \mathcal{S_\mathrm{res}} \cup \{ C \}$
}
}
\Return{$\mathcal{S}_\mathrm{res}$}\;
\caption{Computation of \meps-meaningful boundaries in image $u$.}
\label{algo:meaningfulBoundaries}
\end{algorithm}

\paragraph{Background model.} Now we shall check the consistency of Definition~\ref{def:nfaContrastedCurve}, namely that, in average, no more than \meps curves are detected by chance. In order to make this assertion more precise (in Proposition~\ref{prop:contrastedCurvesNFA} below) we need to define the (\emph{a contrario}) statistical background model that is used to present random input images to the boundary detector.
Following \cite{cao2005,dmm01} we do not directly introduce a statistical image model, but we only state the statistical properties that each level line $C$ in the input set $E$ of level lines should satisfy.
The actual shape of the curve does not matter. We only require that a random gradient value $|Du|(x_i)$ be associated to each of the $n$ regularly sampled points $x_0, x_1, \dots, x_{n-1}$ of $C$, that these $n$ random variables be independent, and with the same distribution $P(|Du|(x_i)>\mu) = H_c(\mu)$. 

\begin{proposition}
\label{prop:contrastedCurvesNFA}
The expected number of DMM \meps-mean\-ing\-ful contrasted boundaries in a random set $E$ of random curves is smaller than \meps, if $E$ follows the above background model.
\end{proposition}

We refer to the work by Cao~\etal~\cite{cao2005} for a complete proof.

Proposition~\ref{prop:contrastedCurvesNFA} allows to interpret the meaningful contrasted curves in Definition~\ref{def:nfaContrastedCurve} within a multi-hypothesis testing framework: namely, the curves detected on an image $u$ are those that allow to reject the null hypothesis (background model) \emph{$\Hy_0$: the values of $|Du|$ are i.i.d., and follow the same distribution as gradient magnitude histogram of the image $u$ itself}.


Definition~\ref{def:nfaContrastedCurve} has some drawbacks. From one side, the use of the minimum or any punctual measure, for the case, can be an unstable measure in the presence of noise. From the other side, it demands the curve to be not likely \emph{entirely} generated by noise (i.e., well contrasted). We already stated that \emph{pieces} of level lines match object boundaries. 
Moreover, as seen on Figure~\ref{fig:conceptMinContrast}, the use of the minimum contrast seems in contradiction with what we perceive. It is therefore too restrictive to impose such a constraint. Since we search for object boundaries, we think the natural model is to select level lines that have well contrasted parts.

\begin{figure}
  \centerline {
      \includegraphics[width=.4\columnwidth]{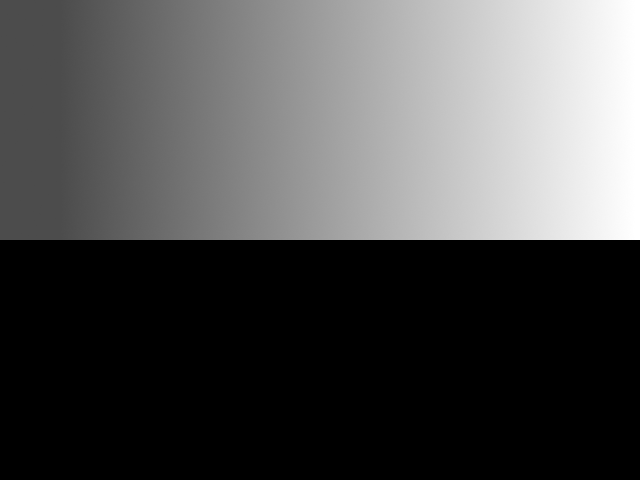}
      \hspace{.2in}
      \includegraphics[width=.4\columnwidth]{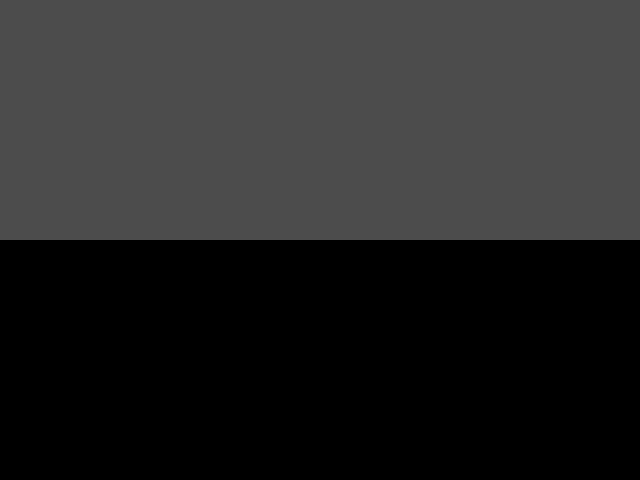}
  }
  
  \caption{Conceptual consequence of using the minimum contrast to detect boundaries. The left image contains a gray gradient and an uniformly black region on its upper and lower halves respectively. The right image is constructed by putting in its upper half the minimum gray level on the left image's upper half. If our perception was tuned to use the minimum contrast to detect the boundary between the two regions, we would perceive that the image on the right is as contrasted as the one on the left, which is clearly not the case.}
  \label{fig:conceptMinContrast}
\end{figure}

\subsection{Partially Contrasted Meaningful Boundaries}

In this direction, we propose to modify the definition of the number of false alarms of a curve, to support a new model where one detects partially contrasted curves. This modification was briefly introduced in~\cite{tepper09msc} and is now explained in detail.

\begin{notation}
 Let $x_0, x_1, \dots, x_{n-1}$ denote $n$ points of a curve $C$ of length $l$. Let $s$ be the mean Euclidean distance between neighboring points. For $x \in C$ denote by $c_i$ ($0 \leq i < n$) the contrast at $x_i$ defined by $c_i = |Du|(x_i)$. We note by $\mu_k$ ($0 \leq k < n$) the $k$-th value of the vector of the values $c_i$ sorted in ascending order.
\end{notation}

For $k \leq N \in \N$ and $p \in [0, 1]$, let us denote by
\begin{equation}
  \bintail (N, k; p) \stackrel{\mathrm{def}}{=} \sum_{j = k}^{N} \binom{N}{j} p^j (1 - p)^{N - j}
\end{equation}
the tail of the binomial law. Desolneux~\etal~present a thorough study of the binomial tail and its use in the detection of geometric structures~\cite{desolneux08}.

The regularized incomplete beta function, defined by 
\begin{equation}
  I(x; a, b) = \frac{\int_0^x t^{a-1} (1-t)^{b-1} dt}{\int_0^1 t^{a-1} (1-t)^{b-1} dt} \text{,}
\end{equation}
is an interpolation $\widetilde{\bintail}$ of the binomial tail to the continuous domain~\cite{desolneux08}:
\begin{equation}
    \widetilde{\bintail} (n, k; p) = I(p; k, n-k+1)
\end{equation}
where $n, k \in \R$. In the case $n$ and $k$ are natural numbers $\widetilde{\bintail} (n, k; p) = \bintail (n, k; p)$.
Additionally the regularized incomplete beta function can be computed very efficiently~\cite{numericalRecipes}.

Following Meinhardt~\etal~\cite{meinhardt08}, for a given curve the probability under $\Hy_0$ that at least $k$ among the $n$ values $c_j$ are greater than $\mu$ is given by the tail of the binomial law $\bintail (n, k; H_c(\mu))$. Thus it is interesting, and more convenient, to extend this model to the continuous case using the regularized incomplete beta function
\begin{equation}
  \widetilde{\bintail} (n \cdot \lsn{s}, k \cdot \lsn{s}; H_c(\mu))
\end{equation}
where $\lsn{s} = \frac{l}{s \cdot n}$ and acts as a normalization factor.
This represents the probability under $\Hy_0$ that, for a curve of length $l$, some parts with total length greater or equal than $\lsn{s} (n-k)$ have a contrast greater than $\mu$.

\begin{definition}
  \label{def:nfaContrastedCurve_k}
  Let $\mathcal{C}$ be a finite set of $N_{ll}$ level lines of $u$. A level line $C \in \mathcal{C}$ is a TMA \meps-meaningful boundary if
  \begin{equation}
  \NFA_K(C) \stackrel{\mathrm{def}}{=} N_{ll}\ K\ \min_{k < K} \widetilde{\bintail} (n \cdot \lsn{2}, k \cdot \lsn{2}; H_c(\mu_k)) < \eps
  \end{equation}
  where $K$ is a parameter of the algorithm. This number is called number of false alarms (NFA) of $C$.
\end{definition}

The parameter $K$ controls the number of points that we allow to be likely generated by noise, that is, a curve must have no more than $K$ points with a ``high'' probability of belonging to the background model. It is simply chosen as a percentile of the total number of points in the curve.
The procedure is similar to Algorithm~\ref{algo:meaningfulBoundaries} but replacing $\NFA$ by $\NFA_K$.

As usual, Definition~\ref{def:nfaContrastedCurve_k} is correct if the following proposition holds.
\begin{proposition}
  \label{prop:nfaContrastedCurve_k}
  The expected number of TMA \meps-mean\-ing\-ful boundaries, in a finite random set $E$ of random curves is smaller than \meps.
\end{proposition}
This very important proof is given in Appendix~\ref{sec:proofNFA_C} to avoid breaking the flow of the discussion.

This new model is an extension of the previous one, since $\NFA_{K=1}(C) = \NFA(C)$. In fact, Definition~\ref{def:nfaContrastedCurve_k} is no other than a relaxation of Definition~\ref{def:nfaContrastedCurve}. We should expect to have new detections and to detect the same lines, with increased stability. This comes from the fact that several punctual measures are used and the minimum is taken over their probability. This was experimentally checked and some results can be seen in Section~\ref{sec:experimentsMeaningful}.

We apply the DMM-MCB and TMA-MCB algorithms to an image of white noise, in order to experimentally check that when $\eps=1$ the number of detections is in average lower than 1. This is confirmed in Figure~\ref{fig:whiteNoise_C}, where the number of detections is actually zero. Even when $\eps=1000$, the number of detections remain very small.

\begin{figure*}
    \centering
%
%
%
    \includegraphics[width=\textwidth]{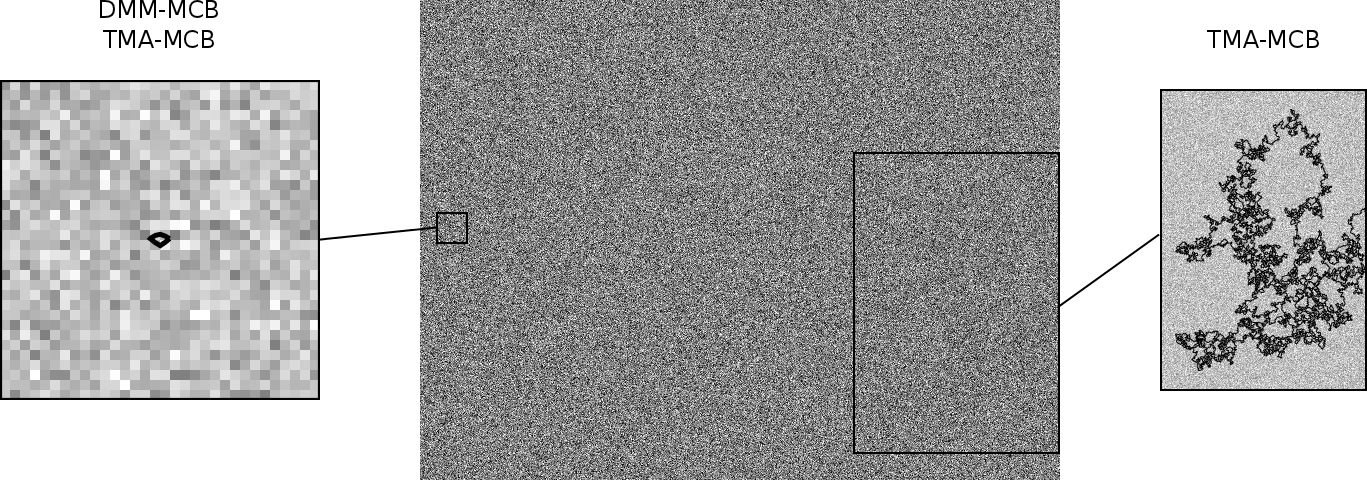}
    \caption{There are 4845004 level lines in the center image of a Gaussian noise with standard deviation 50. By setting $\eps=1000$, DMM-MCB detects one boundary (left detail) and TMA-MCB detects two boundaries (left and right details). At $\eps=1$, both methods detect zero boundaries.}
    \label{fig:whiteNoise_C}
\end{figure*}

In~\cite{cao2005}, other modifications are proposed to the basic meaningful boundaries algorithm.
On the one hand, meaningfulness is computed locally. We will not discuss this further, since we are only interested in the redefinition of the NFA and its consequences. In any case, our redefined NFA can also be used in the same local detection process.
On the other hand, only level lines that remain stable across several zoom scalings are detected. The reason behind this approach is to counter the effect of small perturbations (i.e., noise) in the image. Our scheme handles naturally this effect by minimizing a probability instead of a punctual measure. This was confirmed in our experiments where multiscale stabilization did not provide any visible improvement.

\subsection{Maximal boundaries}

Because of interpolation, meaningful boundaries usually appear in parallel and redundant groups, called bundles. Since the meaningful level lines inherit the tree structure of the topographic map, Desolneux~\etal~\cite{desolneux08} use this structure to efficiently remove redundant boundaries. From now on, we work on the tree composed only of meaningful boundaries.

\begin{definition}
\textnormal{(Monasse and Guichard~\cite{monasse00})}
A monotone section of a level lines tree is a part of a branch such that each node has a unique son and where grey level is monotone (no contrast reversal). A maximal monotone section is a monotone section which is not strictly included in another one.
\end{definition}

\begin{definition}
\textnormal{(Desolneux~\etal~\cite{dmm01})}
A meaningful boundary is maximal meaningful if it has a minimal NFA in a maximal monotone section.
\end{definition}
Algorithm~\ref{algo:newMeaningfulBoundaries} depicts the overall proposed procedure.
\begin{algorithm}[t]
    \SetKwInOut{Input}{input}\SetKwInOut{Output}{output}
    
    \Input{An image $u$, a scalar \meps an integer $K$.}
    \Output{A set of closed curves $\mathcal{S}_\mathrm{res}$.}

    $\mathcal{S} \gets \mathrm{FLST}(u)$\tcp*{Compute the set of level lines}
    $N_{ll} \gets \# \{ \mathcal{S} \}$\;
    Compute the tail histogram $H_c$ of $|Du|$\;
    $\mathcal{S}_\mathrm{res} \gets \emptyset$\;
    \For{$C \in \mathcal{S}$}{
        Compute the length $l$ of $C$\;
        $n \gets \# \{ x \in C\}$\;
        $\mu_1, \dots, \mu_K \gets$ the $K$ smallest values of $|Du|(x)$, $x \in C$\;
        $\displaystyle \mathrm{nfa}_C \gets N_{ll} K \min_{k < K} \widetilde{\bintail} (\tfrac{l}{2}, k \cdot \tfrac{l}{2n}; H_c(\mu_k))$\;

        \lIf{$\mathrm{nfa}_C < \eps$}{
            $\mathcal{S}_\mathrm{res} \gets \mathcal{S}_\mathrm{res} \cup \{ C \}$
        }
    }
    \tcp{Maximality-based pruning:}
    \Repeat{all monotone sections have been explored}{
        Find an unexplored monotone section $\mathcal{S}_\mathrm{M}$ in the level lines tree\;
        $\displaystyle C_\mathrm{M} \gets \max_{C \in \mathcal{S}_\mathrm{M}} \mathrm{nfa}_C$\;
        \For{$C \in \mathcal{S}_\mathrm{M}$}{
            \lIf{$C \in \mathcal{S}_\mathrm{res}$ \textbf{and} $C \neq C_\mathrm{M}$}{
                $\mathcal{S}_\mathrm{res} \gets \mathcal{S}_\mathrm{res} \setminus \{ C \}$
            }
        }
    }

    \Return{$\mathcal{S}_\mathrm{res}$}\;
    \caption{Computation of maximal TMA \meps-meaningful boundaries in image $u$.}
    \label{algo:newMeaningfulBoundaries}
\end{algorithm}

Figure~\ref{fig:buildingSequence} shows an example of the reduction of the number of level lines caused by the maximality constraint. Parallel level lines are eliminated, leading to ``thinner edges .''

\begin{figure*}
  \centerline{
    \includegraphics[width=.3\textwidth]{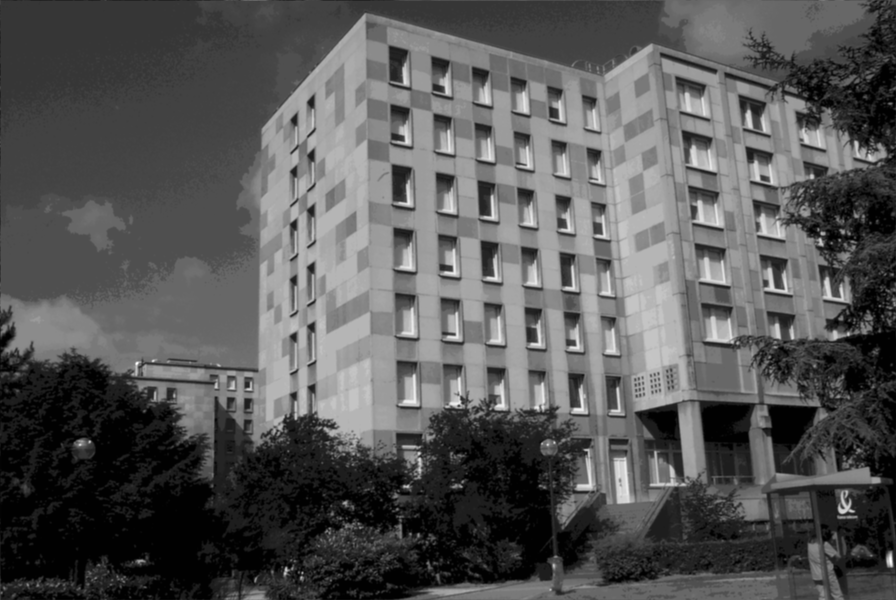}
    \hfil
    \fbox{\includegraphics[width=.3\textwidth]{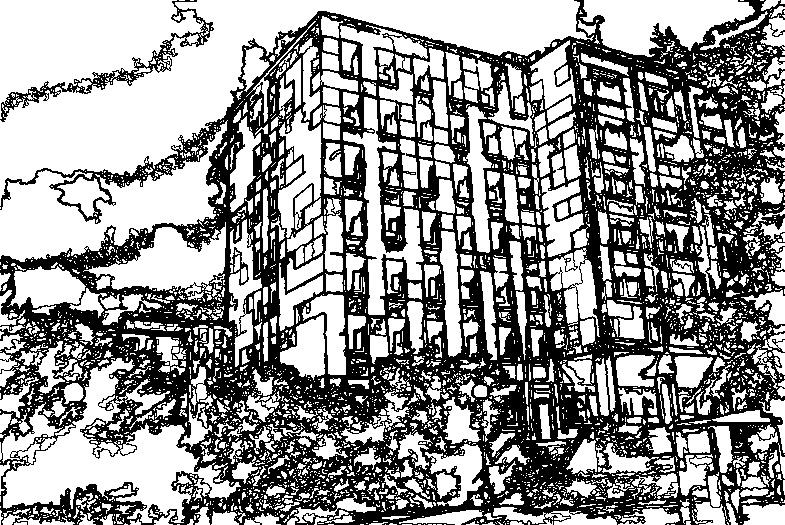}}
    \hfil
    \fbox{\includegraphics[width=.3\textwidth]{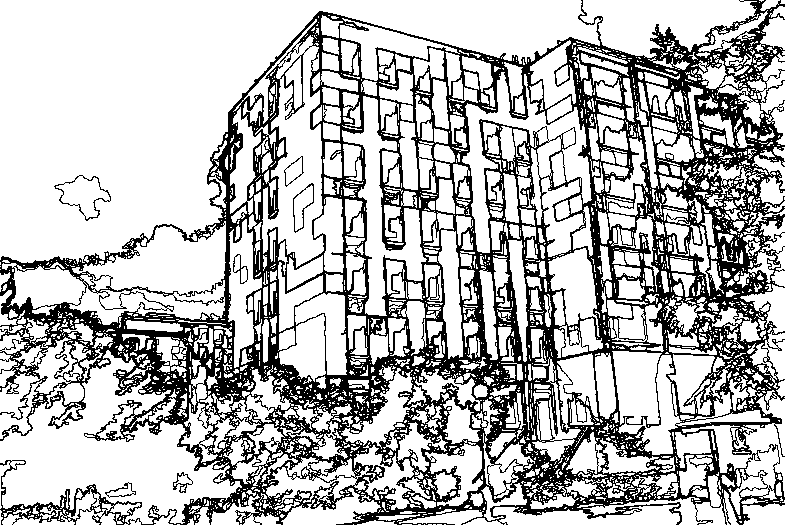}}
  }
  \caption{Effect of the maximality condition over the meaningful boundaries of an image. On the left, original image; on the center, DMM-MCB (8987 lines found); on the left, maximal DMM-MCB (517 lines found).}
  \label{fig:buildingSequence}
\end{figure*}

In the following, when we refer to meaningful boundaries, both in its DMM or TMA versions, we always compute maximal meaningful boundaries.

Notice that working with representative curves of monotone sections has some well-known dangers for
particular configurations that rarely occur in practice.  For example, if the input image contains successively nested objects of different increasing shades of gray, the proposed algorithm will detect only one object of each nested set.  Other definitions that explore local maxima of some saliency measure along the tree, such as MSER~\cite{matas02-mser}, can be used to correct this issue.

Desolneux~\etal~\cite{dmm01} also proposed an algorithm called meaningful edges which aims at detecting salient (i.e., well contrasted) pieces of level lines. TMA-MCB can be considered a hybrid of meaningful boundaries and meaningful edges and presents advantages from both algorithms.
Pieces of level lines belonging to different level lines cannot be compared, since they can have different positions and lengths. This means that we cannot compute maximal meaningful edges in the level lines tree. The TMA-MCB algorithm is able to detect partially salient curves while retaining compatibility with the maximality in the tree. On the other side, it is possible to compute maximal meaningful edges inside a given curve. TMA-MCB, as a provider of the supporting level lines, can be considered a first step towards finding meaningful edges that are maximal in both directions: in the tree, i.e., orthogonal to the curve, and along the curve. The extraction of the optimal pieces in a curve is discussed by Tepper \etal~\cite{tepper12ps}.

\subsection{Practical implications of the change in the NFA}
\label{sec:experimentsMeaningful}

We now address the following question: is there a fundamental difference in practice between DMM-MCB and TMA-MCB? The answer is that, given an image, this change implies noticeable differences in the detected curves. Indeed, TMA-MCB are more robust since the NFAs attained are much lower. Taking the minimum of probabilities is also more stable than taking the minimum on any punctual measure, see Figure~\ref{fig:comparisonNFAunderNoise}.

\begin{figure*}
  \centering
  \begin{tabular}{@{\hspace{0pt}}c@{\hspace{4pt}}c@{\hspace{12pt}}c@{\hspace{12pt}}c@{\hspace{0pt}}}
    & \textsc{image}
    & \textsc{dmm-mcb}
    & \textsc{tma-mcb}  \tabularnewline
    
    \raisebox{.4in}{\begin{sideways}\textsc{original}\end{sideways}} &
    \includegraphics[width=1.4in]{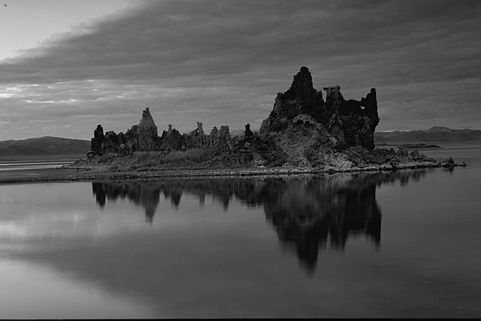} &
    \fbox{\includegraphics[width=1.4in]{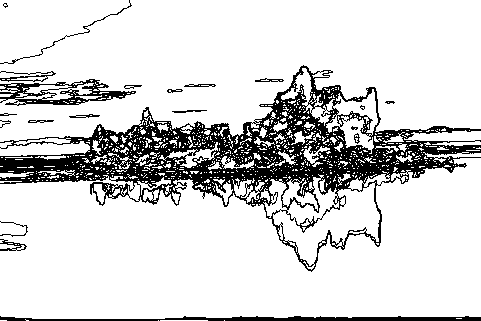}} &
    \fbox{\includegraphics[width=1.4in]{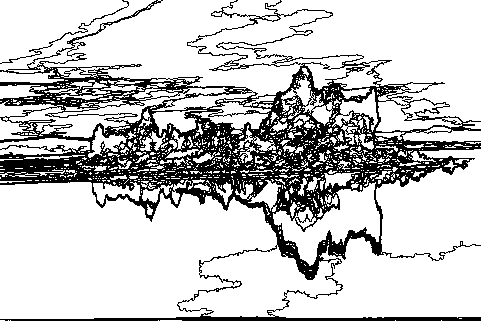}} \tabularnewline
    
    \vspace{4pt}
    \raisebox{.3in}{\begin{sideways}\textsc{original+noise}\end{sideways}} &
    \includegraphics[width=1.4in]{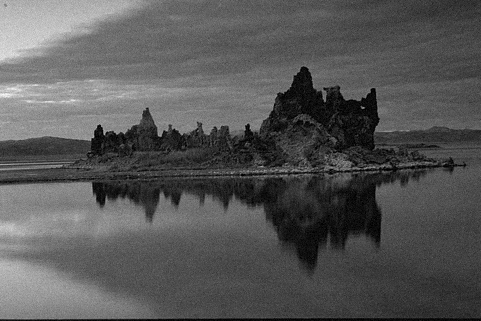} &
    \fbox{\includegraphics[width=1.4in]{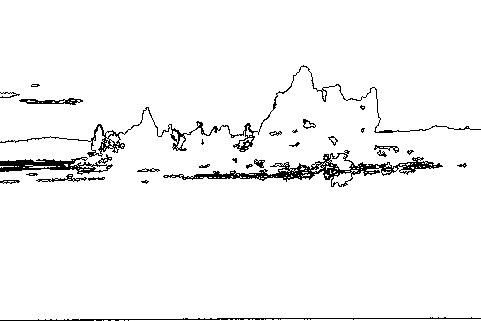}} &
    \fbox{\includegraphics[width=1.4in]{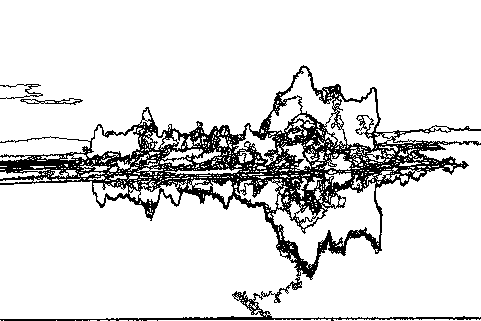}}
  \end{tabular}
  
  \caption{Noise contamination example. The image on the bottom left is contaminated by a small amount of noise. DMM-MCB takes a minimum of punctual measures, thus its result is affected. On the counterpart, result with TMA-MCB is less affected, as it deals with probabilities. Notice that here no smoothing is performed previous to detection, contrarily to the original implementation of the meaningful boundaries algorithm~\cite{dmm01}.}
  \label{fig:comparisonNFAunderNoise}
\end{figure*}

In some cases, by relaxing the meaningfulness threshold in DMM-MCB, that is setting $\eps > 1$, visually better results can be achieved. More level lines are kept, but at the expense of having lower confidence on them. The key advantage with TMA-MCB is that, for a given threshold for \meps, less visually salient level lines are discarded.

One of the possible arguments against TMA-MCB could be that it is no more than a shift of the threshold on the NFA of DMM-MCB. Specifically, that there exists a threshold $\eps' > \eps$ for which DMM-MCB and $\eps'$ would be the same as TMA-MCB and \meps. However, the assertion is clearly false, as shown in Figure~\ref{fig:comparisonNFA}.

\begin{figure*}
  \centering
  \begin{tabular}{@{\hspace{4pt}}c@{\hspace{4pt}}c@{\hspace{4pt}}c@{\hspace{4pt}}c@{\hspace{4pt}}}
    \textsc{image} &
    \textsc{dmm-mcb} \begin{footnotesize}($\eps=10^{-10}$)\end{footnotesize} &
    \textsc{dmm-mcb} \begin{footnotesize}($\eps=1$)\end{footnotesize} &
    \textsc{tma-mcb} \begin{footnotesize}($\eps=10^{-10}$)\end{footnotesize} \tabularnewline

    \includegraphics[width=1.1in]{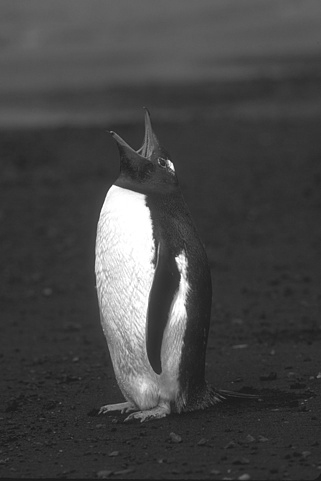} &
    \fbox{\includegraphics[width=1.1in]{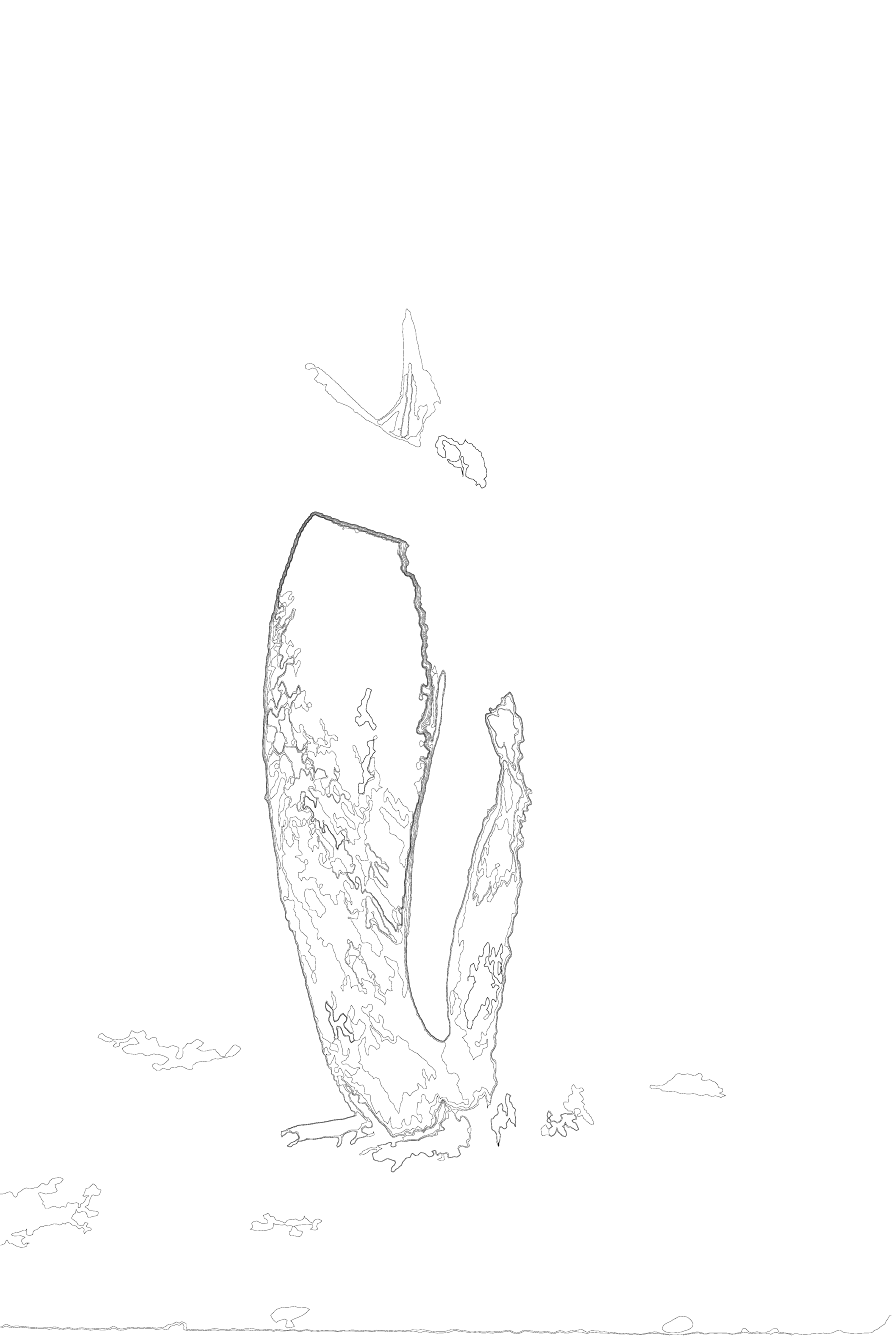}} &
    \fbox{\includegraphics[width=1.1in]{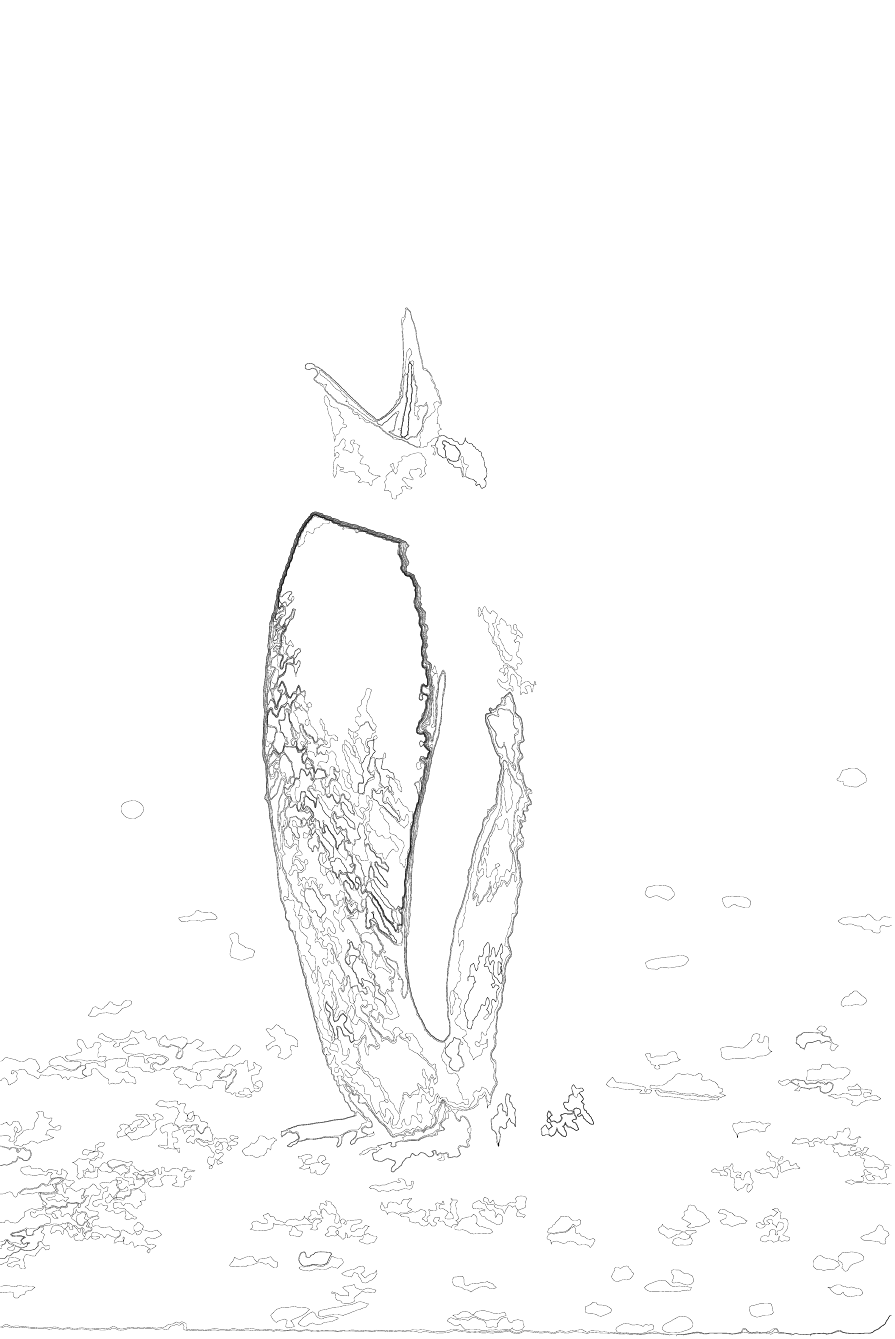}} &
    \fbox{\includegraphics[width=1.1in]{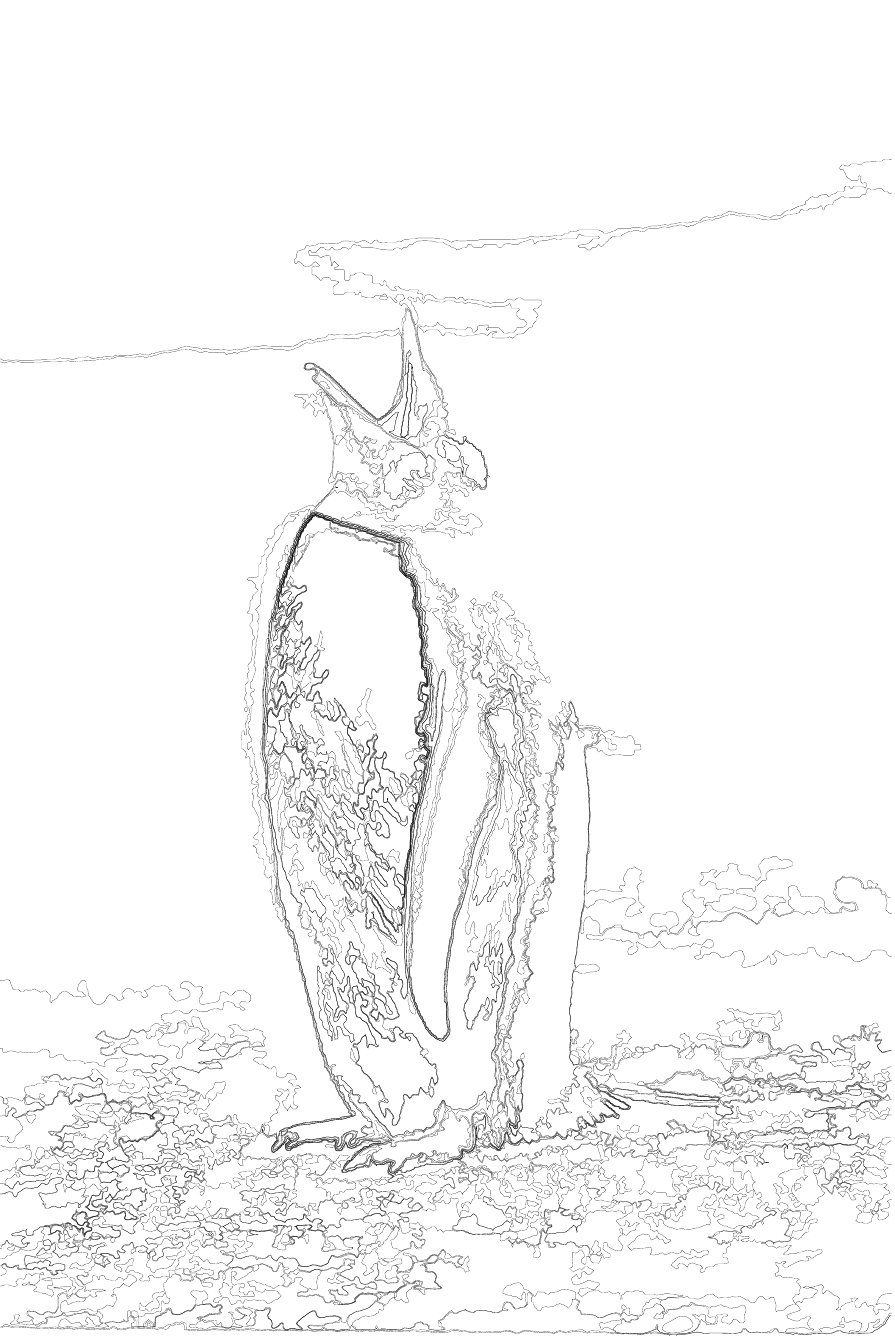}}
  \end{tabular}
  
  \caption{Definition~\ref{def:nfaContrastedCurve_k} is not merely a shift of the threshold on the NFA from  Definition~\ref{def:nfaContrastedCurve}: even relaxing the threshold to its limit ($\eps = 1$), the result with the old method remains roughly the same. A lot of structure missed with Definition~\ref{def:nfaContrastedCurve} is recovered with Definition~\ref{def:nfaContrastedCurve_k}.}
  \label{fig:comparisonNFA}
\end{figure*}

In many applications (e.g., scene reconstruction, image matching), underdetection is far more dangerous than overdetection. Losing structure is critical as it can end-up in a total failure. Detection noise can always be handled (or even tolerated) when the amount of noise does not occlude information, as in our case. TMA-MCB has an advantage over DMM-MB in this respect\footnote{Note however that overdetection might have as well a huge detrimental impact in other applications.}. This is experimentally checked in all examples, even if the difference is more striking in some examples than in others.

Figure~\ref{fig:epsilonEvolution} shows the numerical robustness attained with TMA-MCB. The visually important boundaries in the image have a much lower NFA with TMA-MCB than with DMM-MCB.

\begin{figure*}
  \centering
  \begin{tabular}{@{\hspace{0pt}}m{.1in}@{\hspace{4pt}}m{1.52in}@{\hspace{4pt}}m{1.52in}@{\hspace{4pt}}m{1.52in}@{\hspace{0pt}}}

        &&\includegraphics[width=1.4in]{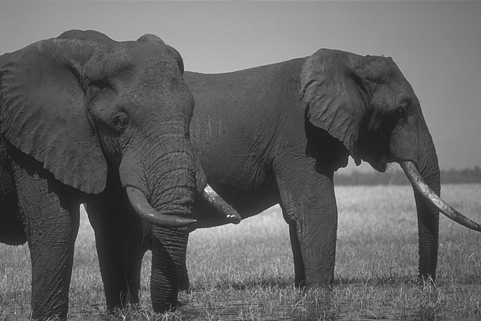}  \tabularnewline

        &
        \centering{\begin{footnotesize}$\eps=10^{-10}$\end{footnotesize}} &
        \centering{\begin{footnotesize}$\eps=10^{-50}$\end{footnotesize}} &
        \centering{\begin{footnotesize}$\eps=10^{-80}$\end{footnotesize}} \tabularnewline
        
        \begin{sideways}\textsc{dmm-mcb}\end{sideways} &
        \fbox{\includegraphics[width=1.4in]{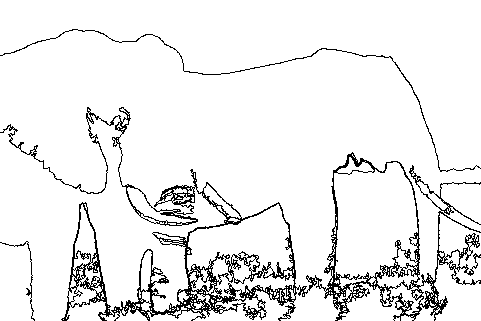}} &
        \fbox{\includegraphics[width=1.4in]{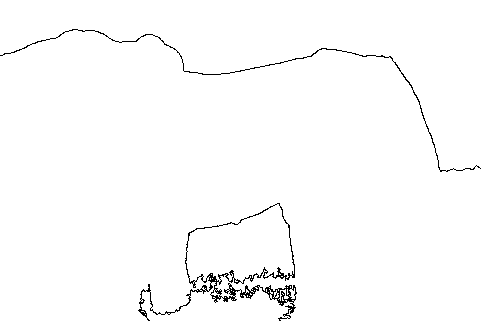}} &
        \fbox{\includegraphics[width=1.4in]{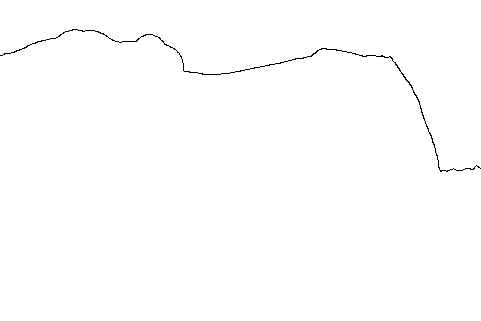}} \tabularnewline
        
        \begin{sideways}\textsc{tma-mcb}\end{sideways} &
        \fbox{\includegraphics[width=1.4in]{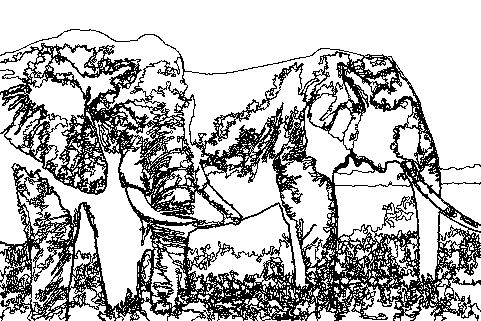}} &
        \fbox{\includegraphics[width=1.4in]{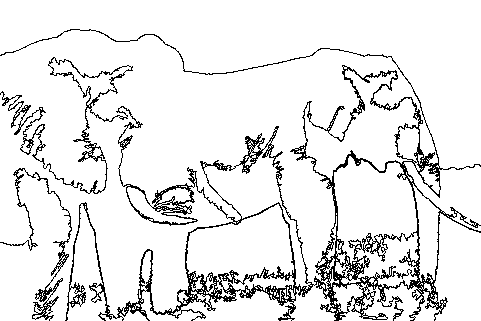}} &
        \fbox{\includegraphics[width=1.4in]{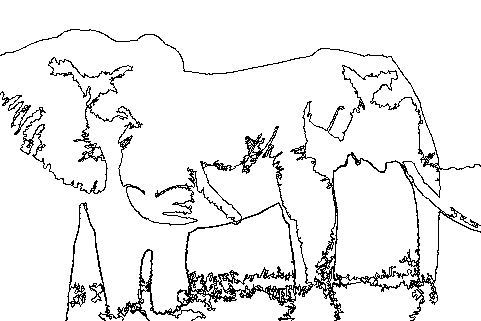}} \tabularnewline
    
      \end{tabular}
  
  \caption{
  Comparison between the stability of DMM-MCB and TMA-MCB. Much lower NFAs are attained with the latter in lines which are visually relevant.
  }
  \label{fig:epsilonEvolution}
\end{figure*}

\section{Combining contrast and good continuation}
\label{sec:meaningfulSmoothBoundaries}

As already stated, in natural images contrasted boundaries often locally coincide with object edges. Thus, they are also incidentally smooth.
Active contours~\cite{kass88} rely on this combination of good contrast and smoothness to provide well localized contours.
In this section, we reprise the work by Cao \etal~\cite{cao2005} and study the possible influence of smoothness in the a contrario detection process.
We conclude that regularity plays an important role in the improvement of the quality of the obtained detections.
This reinforcement phenomenon and the fact that each partial detector can detect most image edges prove a contrario that contrast and regularity are not independent in natural images.

Let $C$ be a rectifiable planar curve, parameterized by its length. Let $l$ be the length of $C$ and $x = C(\tau) \in C$. With no loss of generality, we assume that $\tau = 0$.
\begin{definition}
  \textnormal{(Cao~\etal~\cite{cao2005})}
  Let $s > 0$ be a fixed positive value such that $2s < l$.
  We call regularity of $C$ at $x$ (at scale $s$) the quantity
  \begin{equation}
    R_s (x) = \frac{\max (|x - C(-s)|, |x - C(s)|)}{s}
  \end{equation}
  where $|x_i - x_j|$ represents the Euclidean distance between $x_i$ and $x_j$.
\end{definition}
Figure~\ref{fig:defregularity} visually explains the pertinence of this definition. Only when one of the subcurves $C((-s, 0))$ or $C((0, s))$ is a line segment, $R_s (x) = 1$; in all other cases $R_s (x) < 1$. When $s$ is small enough, regularity is inversely proportional to the curve's curvature around $x$~\cite{cao2005}.

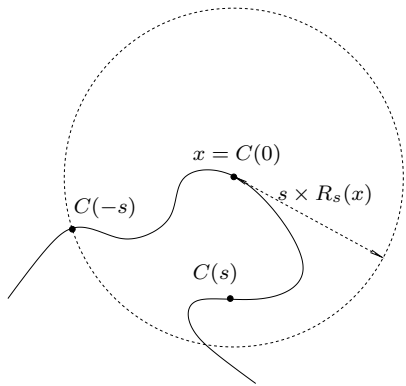
\begin{figure}
  \centering
  \includegraphics[width=150pt]{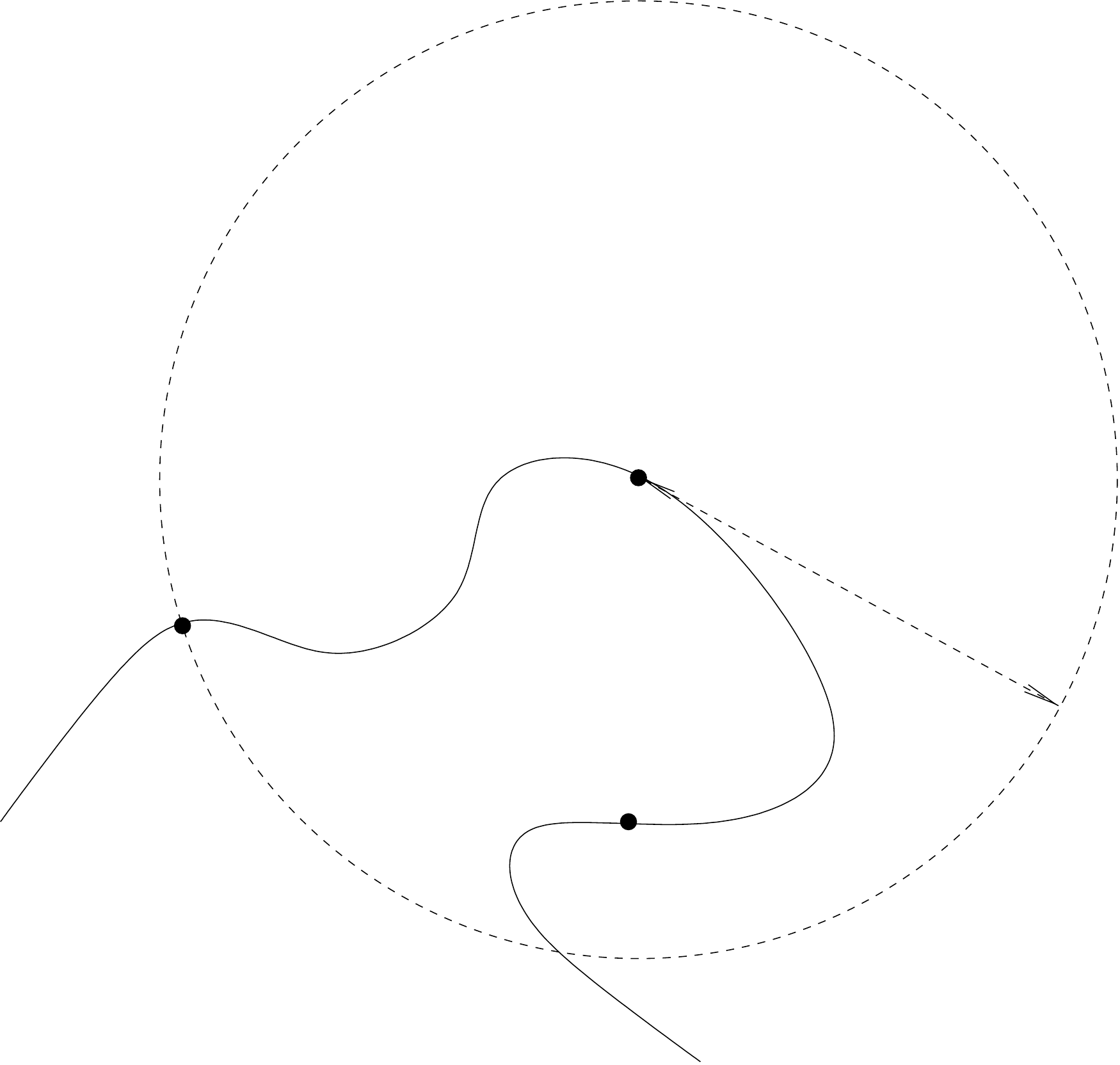}
  \put(-80, 85){$x=C(0)$}
  \put(-125, 65){$C(-s)$}
  \put(-80, 40){$C(s)$}
  \put(-48, 70){$s \times R_s(x)$}
  
  \caption{Reproduced from the work by Cao~\etal~\cite{cao2005}. The regularity at $x$ is obtained by comparing the radius of the circle with $s$. The radius is equal to $s$ if and only if the curve is a straight line. If the curve has a large curvature, the radius will be small compared to $s$.}
  \label{fig:defregularity}
\end{figure}

The question about the choice of $s$ arises naturally and was studied in detail by Cao~\etal~\cite{cao2005} and Mus\'e~\cite{museThesis}. We will limit ourselves to state that a larger value of $s$ (thus at less local scale of analysis) is more robust to noise.
On the other side, $s$ should not be too large either. In practice, and following Cao~\etal~\cite{cao2005} one may safely set $s=5$, which is the value we use in our experiments.

Let us denote by $H_s (r)$ the distribution of the regularity in white noise level lines, i.e.,
\begin{equation}
  H_s (r) = P \Big( R_s (x) > r,\, x \in C,\, C \text{ is a white noise level line} \Big) \text{,}
\end{equation}
which depends only on $s$ and can be empirically estimated.

Again, the curve detection algorithm consists in adequately rejecting the null hypothesis \emph{$\Hy_0$: the values of $|R_s|$ are i.i.d., extracted from a noise image}. We assume that, in the background model, contrast and regularity are independent.

Let us forget for the moment the issues associated with the use of extremal (the minimum) statistics, discussed in Section~\ref{sec:meaningfulBoundaries}.

\begin{definition}
  Let $C$ be a level line in a finite set $\mathcal{C}$ of $N_{ll}$ level lines of image $u$. Let
  \begin{align*}
    \mu &= \min_{x \in C} |Du|(x) \text{,}\\
    \rho &= \min_{x \in C} R_s(x)
  \end{align*}
  be respectively the minimal quantized contrast and regularity along $C$. 
  The level line $C$ is a DMM \meps-meaningful regular boundary (DMM-MRB) if
  \begin{equation}
    \NFA^{R} (C) \stackrel{\mathrm{def}}{=} N_{ll}\ H_s (\rho)^{l / 2 s} < \eps \text{.}
  \end{equation}
  The level line $C$ is a DMM \meps-meaningful contrasted regular boundary (DMM-MRB) if
  \begin{equation}
    \NFA^{\mathrm{CR}} (C) \stackrel{\mathrm{def}}{=} N_{ll} \max \left( H_c (\mu)^{l},\ H_s (\rho)^{l/s} \right) < \eps \text{.}
  \end{equation}
  \label{def:nfaSmoothCurve}
\end{definition}

\begin{remark}
  Cao~\etal~\cite{cao08theory} provided the following definition of meaningful contrasted regular boundaries:
  \begin{equation}
    \NFA^{\mathrm{CR}} (C) \stackrel{\mathrm{def}}{=} N_{ll}\ H_c (\mu)^{l/2}\ H_s (\rho)^{l / 2 s} < \eps \text{.}
  \end{equation}
  Unfortunately, they do not prove that the expected number of \meps-meaningful contrasted regular boundaries in a finite set of random curves is smaller than \meps. This fact is annoying since the threshold \meps is emptied of meaning. It is not by any means an easy proof and we have not found a solution yet. However, we have proven that by slightly changing their definition in the following manner
\begin{equation}
  \NFA^{\mathrm{CR}} (C) \stackrel{\mathrm{def}}{=} N_{ll}\ H_c (\mu)^{{l}^2 / 2 s}\ H_s (\rho)^{{l}^2 / 2 s} \text{.}
  \label{eq:nfaSmoothCurve2}
\end{equation}
a proof can be built~\cite{tepperPhD}.

Although theoretically sound, meaningful contrasted regular boundaries defined by Equation~\ref{eq:nfaSmoothCurve2} do not provide satisfactory results. This is a consequence of using the exponent $l^2$. With respect to DMM-MCB (Definition~\ref{def:nfaContrastedCurve}, p.~\pageref{def:nfaContrastedCurve}) and even if the regularity term has high probability (say one), raising the contrast term to a much larger power will shift the NFA of all curves towards zero. Irregular curves that were not meaningful by their contrast, might become meaningful regular boundaries. This is certainly an unwanted side effect.
\leavevmode\unskip\penalty9999 \hbox{}\nobreak\hfill \quad\hbox{$\triangle$}
\end{remark}

Definition~\ref{def:nfaSmoothCurve} exhibits some interesting properties:
\begin{itemize}
  \item A contrasted but irregular curve will not be detected;
  \item A regular but non-contrasted curve will not be detected;
  \item An irregular and non-contrasted curve will not be detected;
  \item A regular and contrasted curve will be detected.
\end{itemize}
Both gestalts, i.e., contrast and good continuation, interact in a novel way: instead of cooperating by reinforcing each other, as in Equation~\ref{eq:nfaSmoothCurve2}, they compete for the ``control'' of the curve. As the exponent in the contrast term is greater than the exponent in the regularity term ($l > l/s$), the contrast term will in general dominate the detections and the regularity will act as an additional sanity check.

The shifting phenomenon mentioned in the above remark will still be present. However, $2l$ is much less aggressive than $l^2$ and its effect will be doubly mitigated: (1) since $l \gg 2$ and (2) because of the controlling effect of using the maximum.

Since TMA-MCB is a relaxed version of DMM-MCB, we profit from such knowledge and also relax the definition of meaningful contrasted regular boundaries. This relaxation will prove particularly relevant for the contrasted regular case.
\begin{definition}
  \label{def:nfaContrastedSmoothCurve_k}
  Let $\mathcal{C}$ be a finite set of $N_{ll}$ level lines of $u$.
  A level line $C \in \mathcal{C}$ is a TMA \meps-meaningful contrasted regular boundary (TMA-MCRB) if
  \begin{equation}
    \NFA_{K}^{\mathrm{CR}}(C) \stackrel{\mathrm{def}}{=} N_{ll}\ K_c\ K_s
    \max \left(
    \begin{split}
    \min_{k < K_c} I_c (C, k)^2 \\
    \min_{k < K_s} I_s (C, k)^2
    \end{split}
    \right) < \eps \text{,}
  \end{equation}
  where
  \begin{align*}
    I_c (C, k) &= \widetilde{\bintail} (n \cdot \lsn{2}, k \cdot \lsn{2}; H_c(\mu_k)) \text{,}\\
    I_s (C, k) &= \widetilde{\bintail} (n \cdot \lsn{2s}, k \cdot \lsn{2s}; H_s(\rho_{k})) \text{,}
  \end{align*}
  and $K_c$ and $K_s$ are parameters of the algorithm. This number is called number of false alarms (NFA) of $C$.
\end{definition}

Here $K_c$ and $K_s$ have the same meaning as $K$ in Definition~\ref{def:nfaContrastedCurve_k} and they are also set as a percentile of the total number of points in the curve.

\begin{proposition}
  The expected number of TMA \meps-mean\-ing\-ful contrasted regular boundaries in a finite set $E$ of random curves is smaller than \meps.
\end{proposition}
This very important proof is given in Appendix~\ref{sec:proofNFA_CR} to avoid breaking the flow of the discussion.

For completeness, we provide the following definition.
\begin{definition}
  \label{def:nfaSmoothCurve_k}
  Let $\mathcal{C}$ be a finite set of $N_{ll}$ level lines of $u$. A level line $C \in \mathcal{C}$ is a TMA \meps-meaningful regular boundary (TMA-MRB) if
  \begin{equation}
    \NFA_K^{\mathrm{R}}(C) \stackrel{\mathrm{def}}{=} \\ N_{ll}\ K_s \min_{k < K_s} \widetilde{\bintail} (n \cdot \lsn{2s}, k \cdot \lsn{2s}; H_s(\rho_{k})) < \eps \text{,}
  \end{equation}
  and $K_s$ is a parameter of the algorithm. This number is called number of false alarms (NFA) of $C$.
\end{definition}

As a sanity check, we apply the DMM-MCRB and TMA-MCRB algorithms to an image of white noise. We would expect that when $\eps=1$ the number of detections is in average lower than 1. This is checked in Figure~\ref{fig:whiteNoise_CR}, where the number of detections is actually zero. Even when $\eps=1000$, the number of detections remain negligible.

\begin{figure*}
    \centering
%
    \includegraphics[width=.7\textwidth]{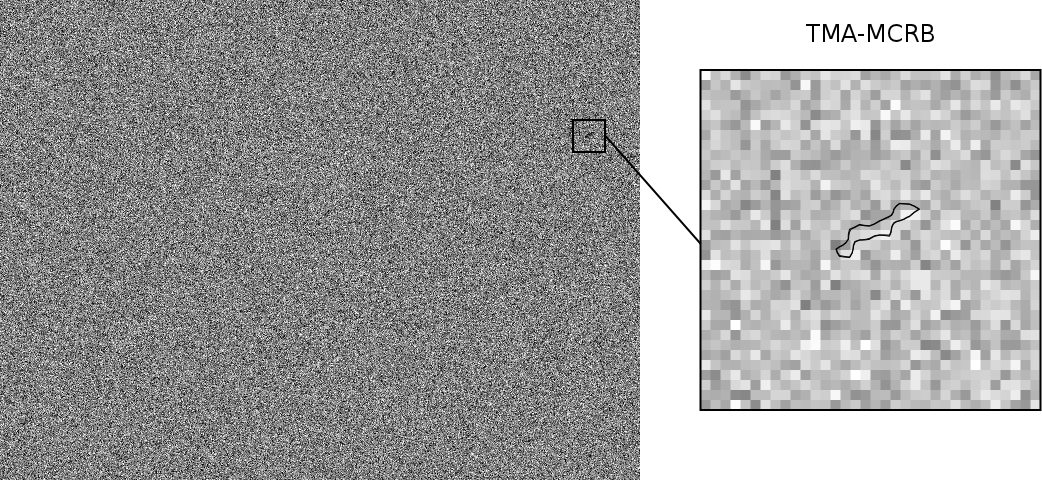}
    \caption{There are 4845004 level lines in the left image of a Gaussian noise with standard deviation 50. By setting $\eps=1000$, DMM-MCRB  detects zero boundaries and TMA-MCRB detects two boundaries forming a packet (right detail). At $\eps=1$, both methods detect zero boundaries.}
    \label{fig:whiteNoise_CR}
\end{figure*}

An immediate objection to the use of regularity might be: since high curvature points are often regarded as very meaningful perceptually~\cite{attneave54}, why such an emphasis in discarding them? The answer is also immediate: we detect \emph{partially} contrasted and regular level lines. Hence, a curve containing a relatively small number of high curvature points will be detected by TMA-MCRB but not by DMM-MCRB. In this scenario, these high curvature points will become more surprising, because of their seldomness, and thus meaningful.

The procedure for finding maximal meaningful regular or contrasted regular boundaries is similar to Algorithm~\ref{algo:newMeaningfulBoundaries}, replacing $\NFA$ by $\NFA_K^\mathrm{R}$ or $\NFA_K^\mathrm{CR}$, respectively.

\subsection{Discussion}

We will now examine the results of the proposed competition between contrast and good continuation.

The benefits of using meaningful contrasted regular boundaries are clear in Figure~\ref{fig:contrastedVSregular}. In both examples, only using contrast produces an overdetection (level lines are detected in areas with texture, e.g. the vegetation on the left, or exhibiting a slight gradient, e.g. the sky and the dome on the right) while only using good continuation produces an underdetection (e.g. the bridge on the left and the bell on the right). The combination of both gestalts corrects the issues by keeping the best from both worlds: most undesired level lines disappear (e.g. the vegetation on the left and the sky on the right) while the desired ones are kept (e.g. the bridge on the left and the bell on the right).

\begin{figure*}
  \centering
  \begin{tabular}{@{\hspace{0pt}}m{.1in}@{\hspace{4pt}}m{.4\textwidth}@{\hspace{4pt}}m{.4\textwidth}@{\hspace{0pt}}}
    \begin{sideways}\textsc{image}\end{sideways} &
    \includegraphics[width=.4\textwidth]{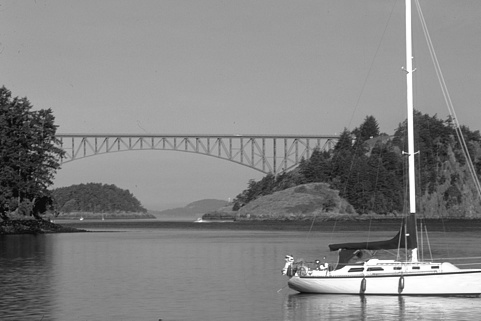} &
    \includegraphics[width=.4\textwidth]{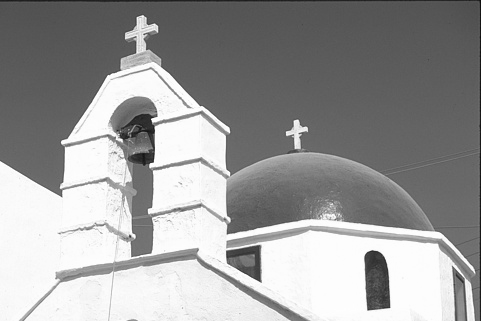} \tabularnewline
    
    \begin{sideways}\textsc{tma-mcb}\end{sideways} &
    \fbox{\includegraphics[width=.4\textwidth]{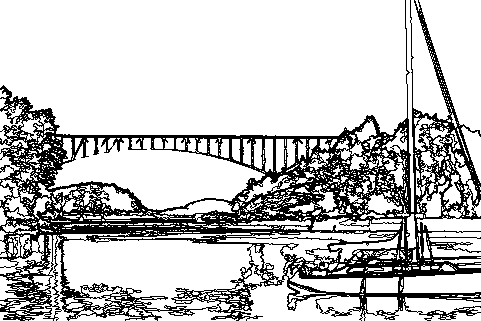}} &
    \fbox{\includegraphics[width=.4\textwidth]{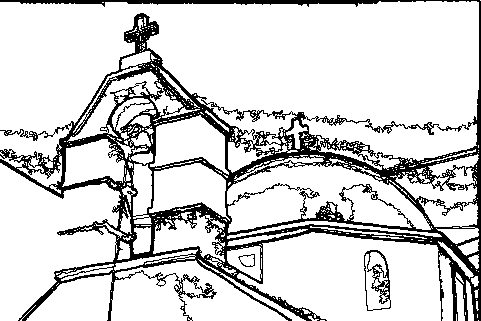}} \tabularnewline

    \begin{sideways}\textsc{tma-mrb}\end{sideways} &
    \fbox{\includegraphics[width=.4\textwidth]{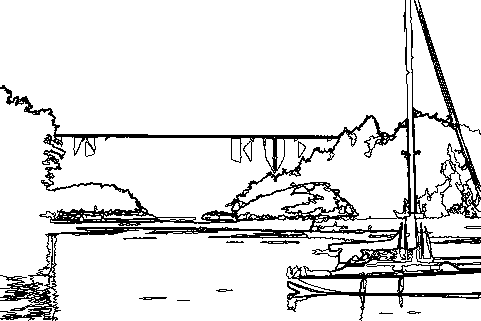}} &
    \fbox{\includegraphics[width=.4\textwidth]{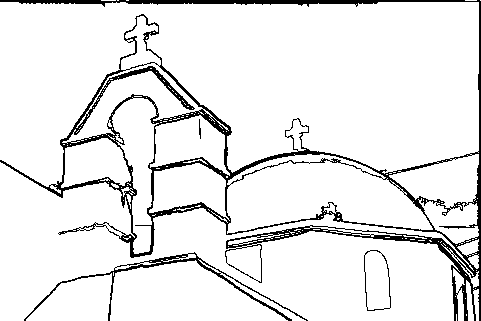}} \tabularnewline

    \begin{sideways}\textsc{tma-mcrb}\end{sideways} &
    \fbox{\includegraphics[width=.4\textwidth]{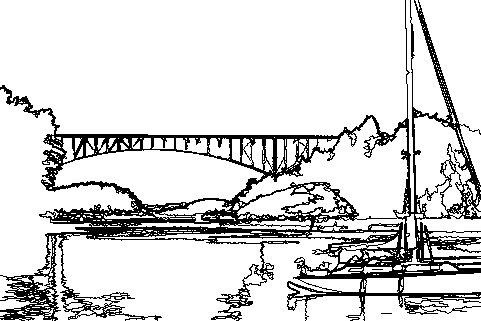}} &
    \fbox{\includegraphics[width=.4\textwidth]{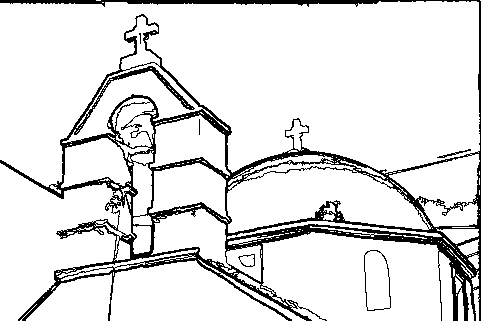}} \tabularnewline
  \end{tabular}

  \caption{Comparison of TMA-MCB (Definition~\ref{def:nfaContrastedCurve_k}), TMA-MRB (Definition~\ref{def:nfaSmoothCurve_k}), and TMA-MCRB (Definition~\ref{def:nfaContrastedSmoothCurve_k}).}
  \label{fig:contrastedVSregular}
\end{figure*}

Although more complicated to analyze, Figure~\ref{fig:star-wars} further supports our claims. See the detail on Harrison Ford's sleeve: it is completely lost by using contrast, partially recovered by using good continuation and well recovered by combining them.

It is important to point out that in general, good continuation has a predominant effect over contrast. In the depicted examples, meaningful contrasted boundaries have lower NFAs than meaningful smooth ones. This explains the visual effect that we perceive when looking at the results: contrasted regular boundaries are basically regular boundaries reinforced by some contrasted parts.

\begin{figure*}
  \centering
  \begin{tabular}{@{\hspace{0pt}}c@{\hspace{12pt}}c@{\hspace{0pt}}}
    \textsc{image} & \textsc{tma-mcb} \tabularnewline
    \includegraphics[width=.4\textwidth]{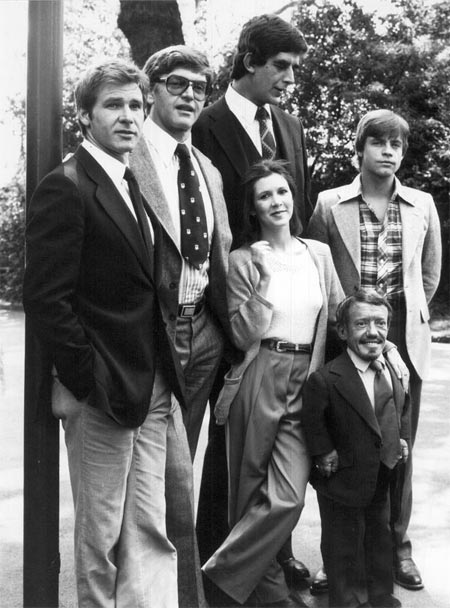} &
    \fbox{\includegraphics[width=.4\textwidth]{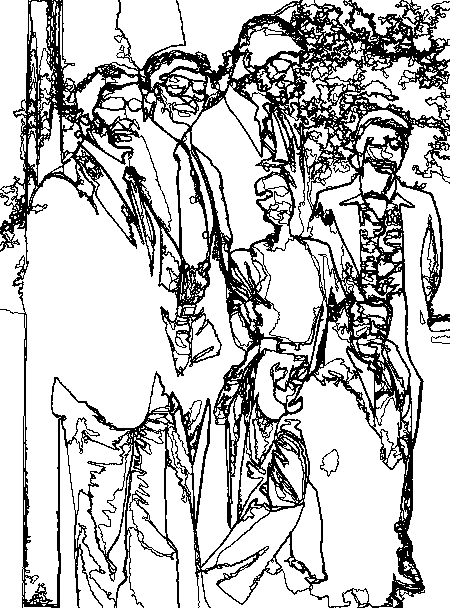}} \tabularnewline

    \textsc{tma-mrb} & \textsc{tma-mcrb} \tabularnewline
    \fbox{\includegraphics[width=.4\textwidth]{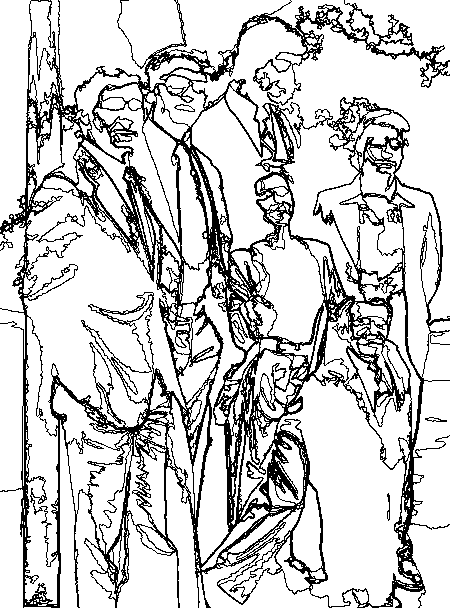}} &
    \fbox{\includegraphics[width=.4\textwidth]{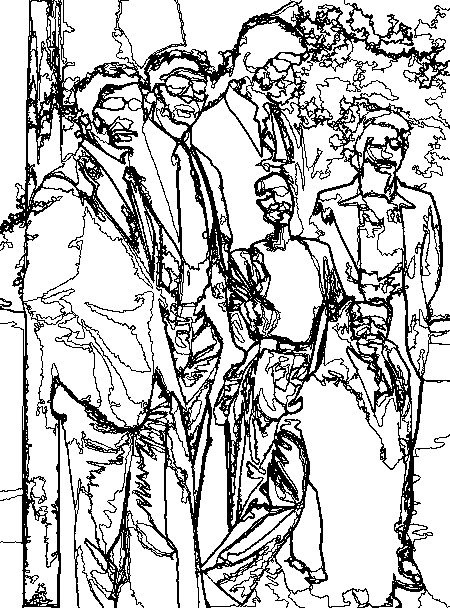}} \tabularnewline
  \end{tabular}

  \caption{Comparison of TMA-MCB (Definition~\ref{def:nfaContrastedCurve_k}), TMA-MRB (Definition~\ref{def:nfaSmoothCurve_k}), and TMA-MCRB (Definition~\ref{def:nfaContrastedSmoothCurve_k}).}
  \label{fig:star-wars}
\end{figure*}

The example in Figure~\ref{fig:watchmen} is a real scene, extremely complicated from the edge detection point of view. In any case, all results are globally satisfactory. Noticeable differences between the methods are perceived by looking at the signs containing letters.

\begin{figure*}
  \centering
  \begin{tabular}{@{\hspace{0pt}}c@{\hspace{4pt}}c@{\hspace{0pt}}}
    \textsc{image} & \textsc{tma-mcb} \tabularnewline
    \includegraphics[width=2.3in]{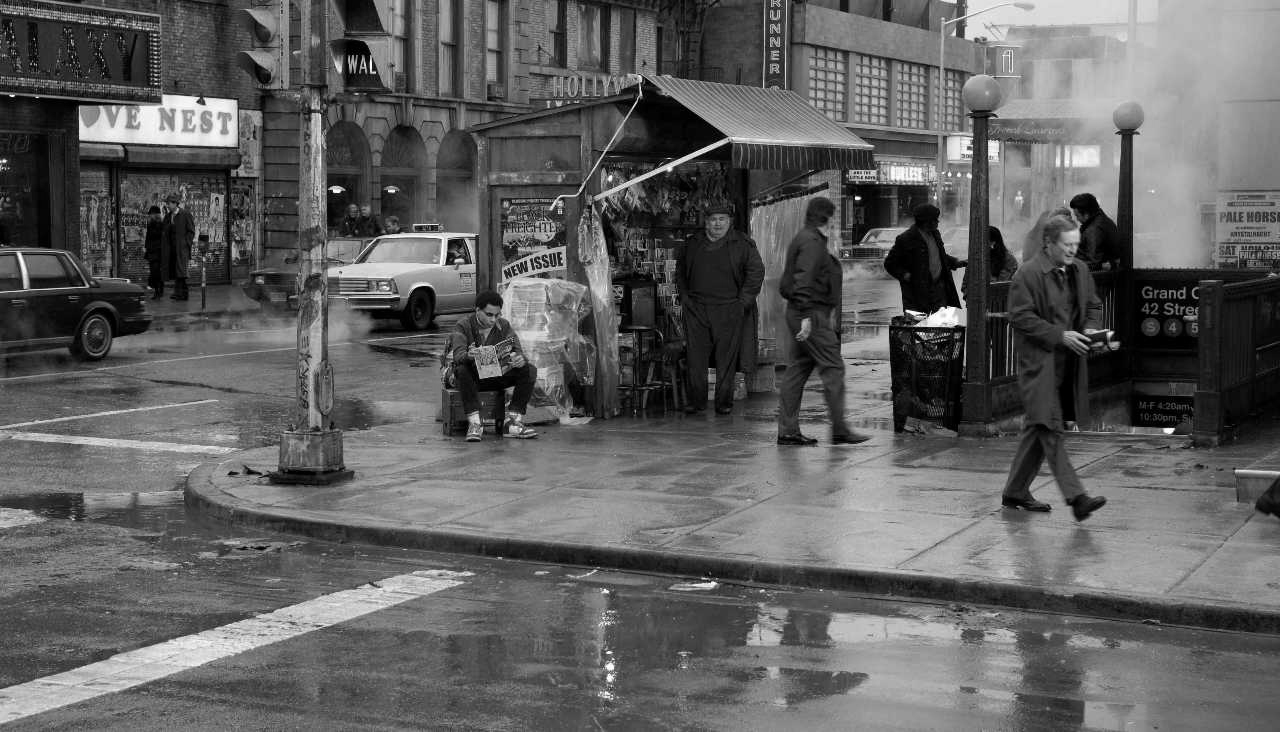} &
    \fbox{\includegraphics[width=2.3in]{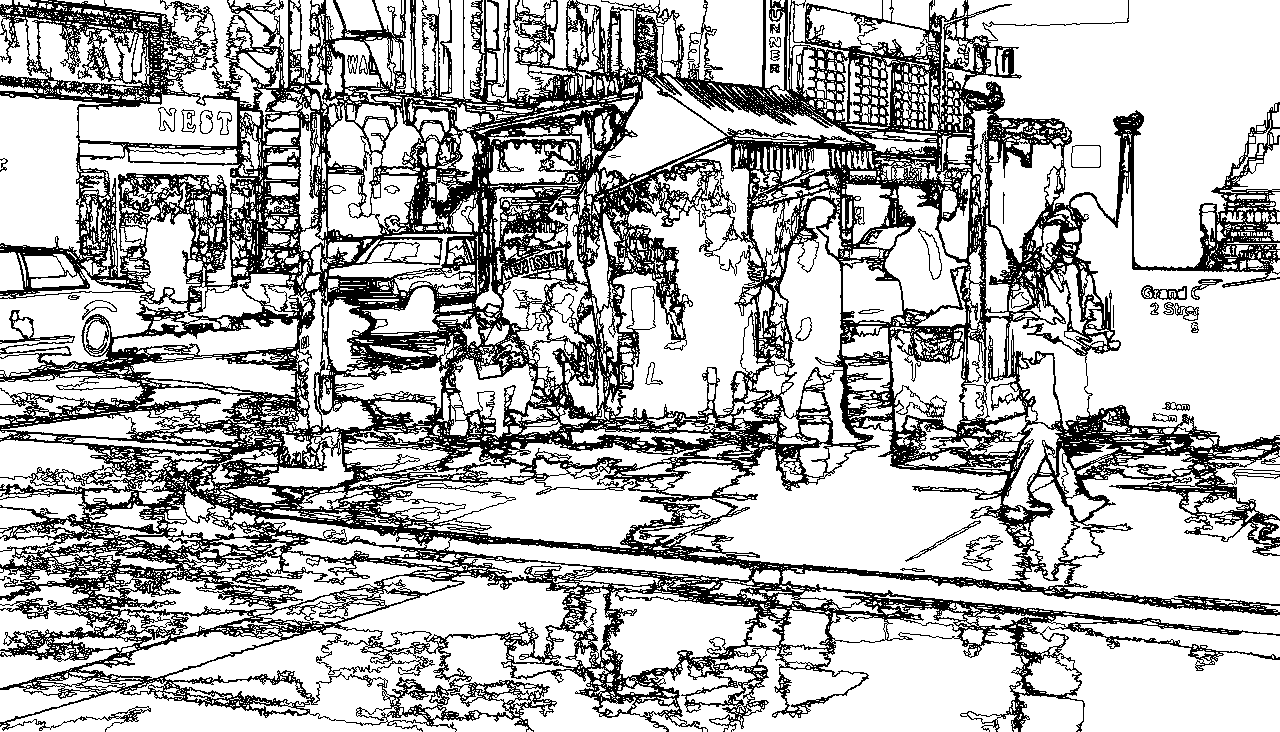}} \tabularnewline
    
    \textsc{tma-mrb} & \textsc{tma-mcrb} \tabularnewline
    \fbox{\includegraphics[width=2.3in]{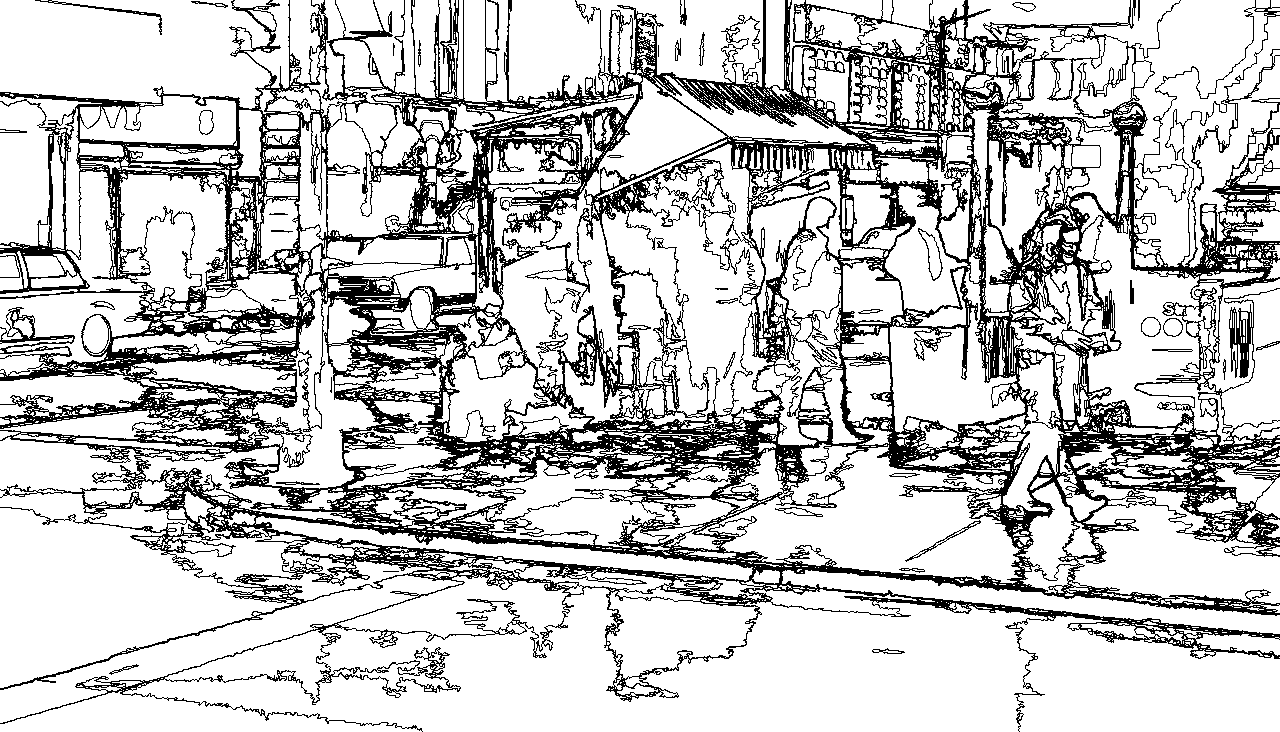}} &
    \fbox{\includegraphics[width=2.3in]{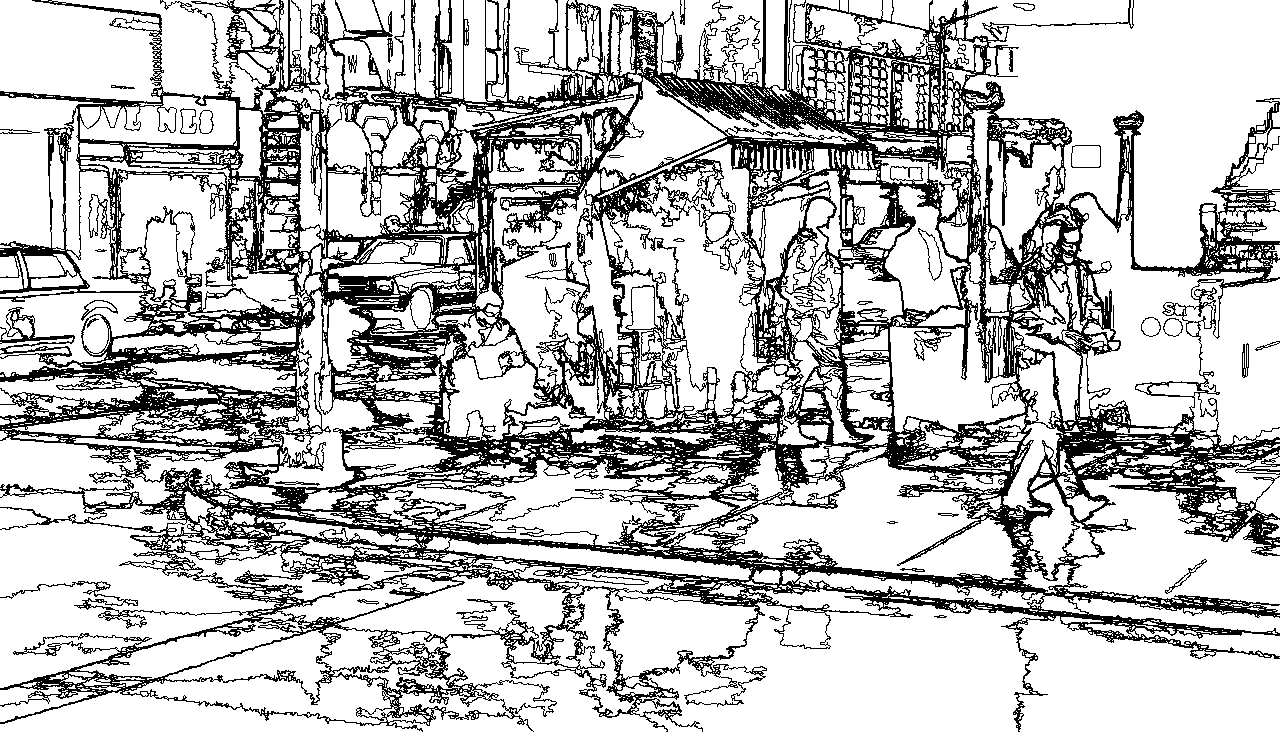}} \tabularnewline
  \end{tabular}
  
  \caption{Comparison of TMA-MCB (Definition~\ref{def:nfaContrastedCurve_k}), TMA-MRB (Definition~\ref{def:nfaSmoothCurve_k}), and TMA-MCRB (Definition~\ref{def:nfaContrastedSmoothCurve_k}).}
  \label{fig:watchmen}
\end{figure*}

We lastly compare TMA-MCRB with DMM-MCRB in Figure~\ref{fig:contrastedRegularDMMvsTMA}. As already stated TMA-MCB often detects more structure than DMM-MCB (second and third rows). This effect is amplified in DMM-MCRB, and can lead to severe underdetections (fourth row). On the other hand, the relaxation present in the TMA version allows to recover the structure more faithfully(fifth row), albeit some mild overdetections.

\begin{figure*}
  \centering
  \begin{tabular}{@{\hspace{0pt}}m{.08in}@{\hspace{4pt}}m{.255\textwidth}@{\hspace{4pt}}m{.255\textwidth}@{\hspace{4pt}}m{.255\textwidth}@{\hspace{4pt}}m{.15\textwidth}@{\hspace{0pt}}}
    \begin{sideways}\textsc{image}\end{sideways} &
    \includegraphics[width=.255\textwidth]{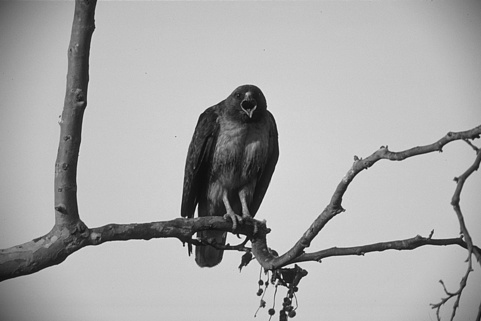} &
    \includegraphics[width=.255\textwidth]{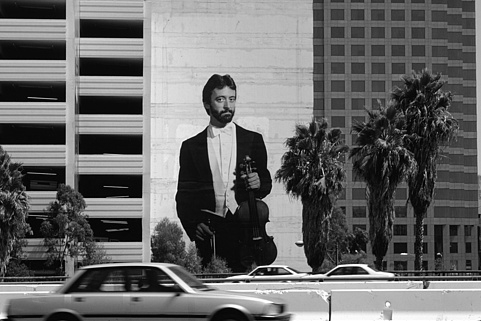} &
    \includegraphics[width=.255\textwidth]{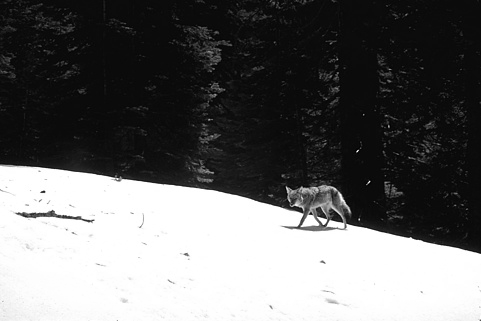} &
    \includegraphics[width=.15\textwidth]{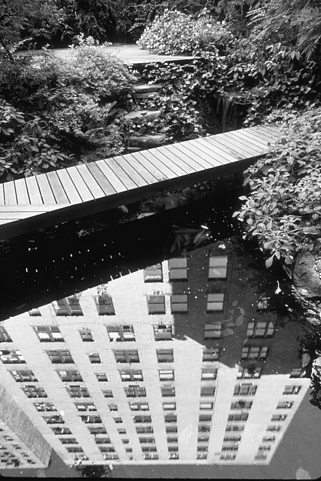} \tabularnewline
    
    \begin{sideways}\textsc{dmm-mcb}\end{sideways} &
    \fbox{\includegraphics[width=.255\textwidth]{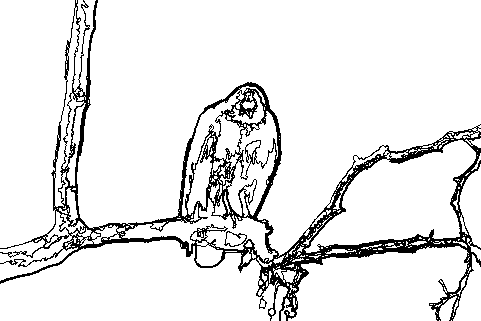}} &
    \fbox{\includegraphics[width=.255\textwidth]{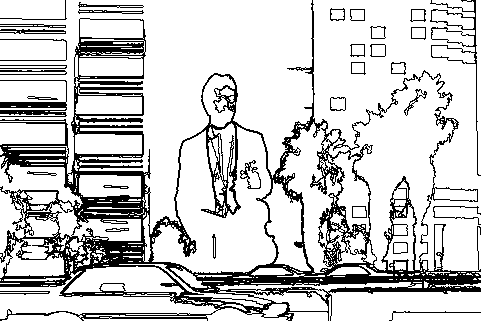}} &
    \fbox{\includegraphics[width=.255\textwidth]{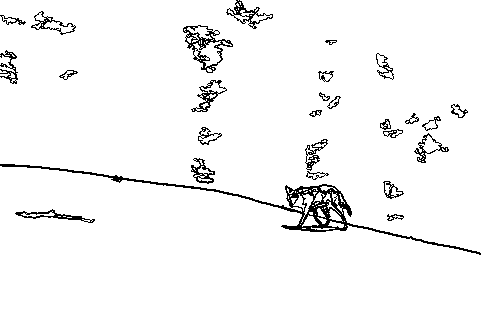}} &
    \fbox{\includegraphics[width=.15\textwidth]{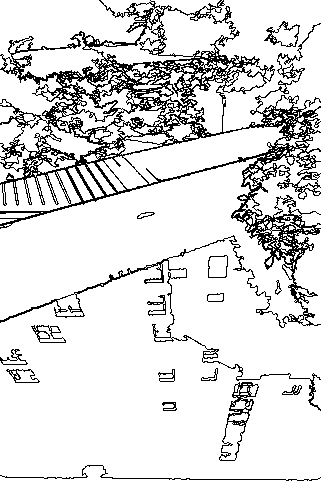}} \tabularnewline

    \begin{sideways}\textsc{tma-mcb}\end{sideways} &
    \fbox{\includegraphics[width=.255\textwidth]{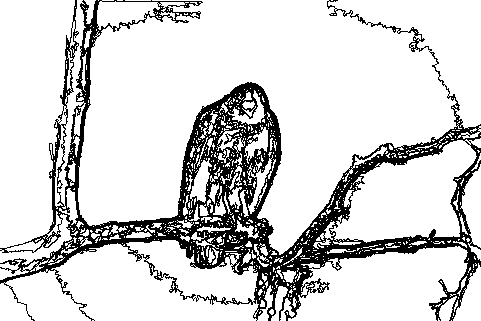}} &
    \fbox{\includegraphics[width=.255\textwidth]{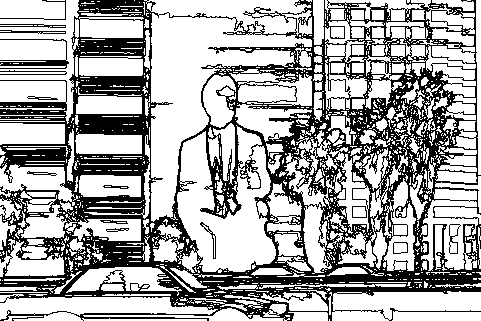}} &
    \fbox{\includegraphics[width=.255\textwidth]{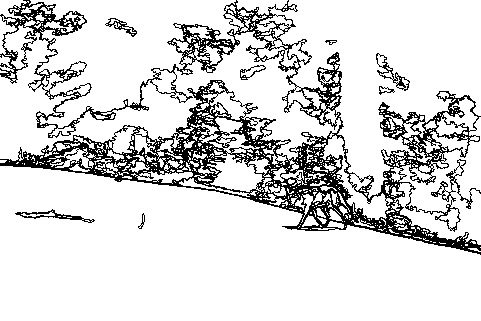}} &
    \fbox{\includegraphics[width=.15\textwidth]{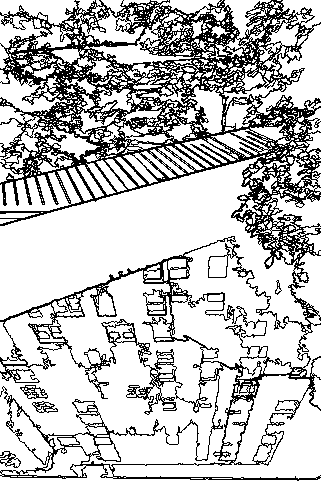}} \tabularnewline


    \begin{sideways}\textsc{dmm-mcrb}\end{sideways} &
    \fbox{\includegraphics[width=.25\textwidth]{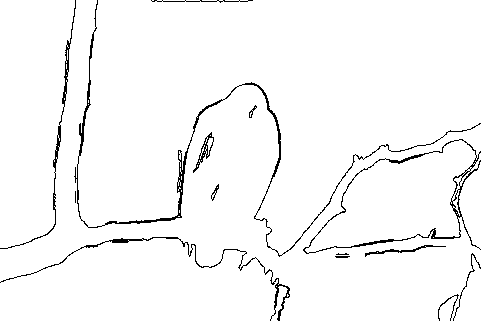}} &
    \fbox{\includegraphics[width=.25\textwidth]{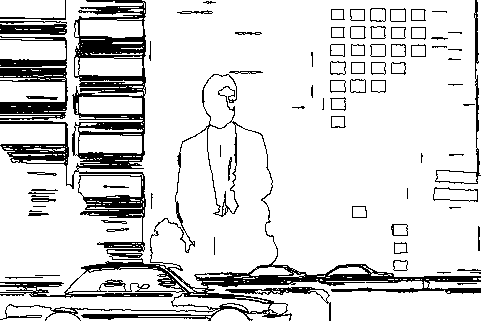}} &
    \fbox{\includegraphics[width=.25\textwidth]{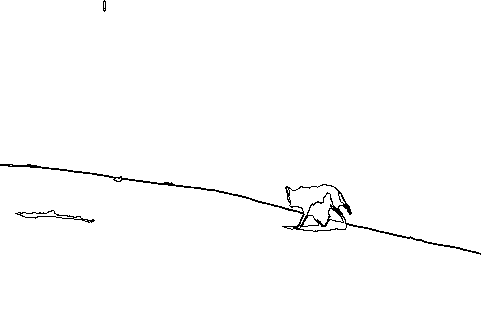}} &
    \fbox{\includegraphics[width=.15\textwidth]{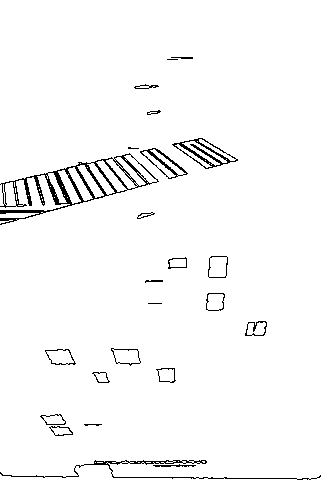}} \tabularnewline

    \begin{sideways}\textsc{tma-mcrb}\end{sideways} &
    \fbox{\includegraphics[width=.25\textwidth]{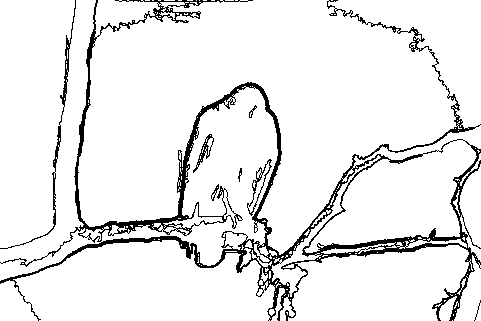}} &
    \fbox{\includegraphics[width=.25\textwidth]{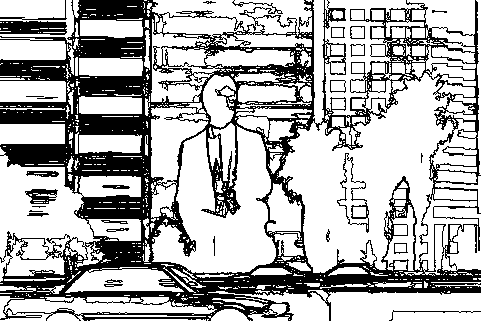}} &
    \fbox{\includegraphics[width=.25\textwidth]{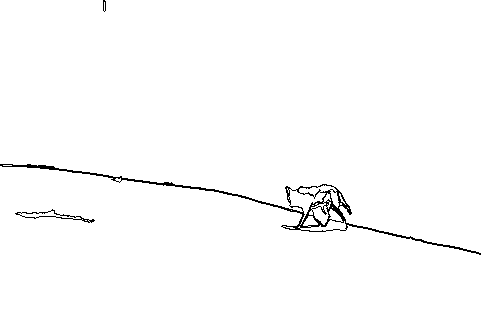}} &
    \fbox{\includegraphics[width=.15\textwidth]{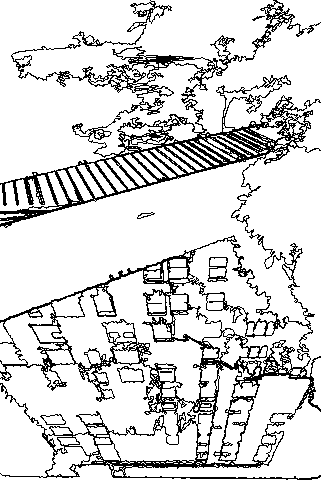}} \tabularnewline

  \end{tabular}

  \caption{
  Comparison of DMM-MCB, TMA-MCB, DMM-MCRB, and TMA-MCRB. DMM-MCRB may produce severe underdetections.}
  \label{fig:contrastedRegularDMMvsTMA}
\end{figure*}

\section{Conclusions}
\label{sec:conclusions}

This work presents a novel contribution to the field of image structure retrieval. We think that the topographic map is an extremely well suited theoretical framework to perform that task. Mathematical Morphology has proved this in depth and extension with the work it developed. In that direction, we based our work on the algorithm called Meaningful Boundaries~\cite{desolneux08}, introducing a few deep modifications that help improve the results.

First, the criterion of meaningfulness was relaxed. In the new definition, a level line can have a non-causal piece and still be considered perceptually important. We also provide an intuitive parameter that allows to deal with the length of that piece.

Second, we analyze the interaction of two fundamental cues for the perception of contours: contrast and regularity. We propose a new way of combining these features in which they compete for the control of the boundary saliency. Experiments show the suitability of this combination strategy.


Examples of the resulting image structure retrieval method were presented, soundly showing that its theoretical advantages are also validated in practice. The proposed method increases significantly the robustness and the stability of the detections.

As a final remark, the maximality constraint presents some issues. All the packets of parallel level line pieces are not eliminated by it. The exploration of another kind of algorithm based on maximality along the gradient direction might help to eliminate this effect~\cite{meinhardt08b}.

\appendix

\section{Proofs}

A classical lemma will be needed in the following.
\begin{lemma}
  Let $X$ be a real random variable. Let $F (x) = \Pr(X \leq x)$ be the repartition function of $X$. Then, for all $t \in (0, 1)$,
  \begin{equation*}
  \Pr(F(X) < t) \leq t \text{.}
  \end{equation*}
  In the same way, let $H(x) = \Pr(X \geq x)$. Then for all $t \in [0, 1]$,
  \begin{equation*}
  \Pr(H(X) < t) \leq t \text{.}
  \end{equation*}
  \label{lem:classic}
\end{lemma}

\subsection{Meaningful Contrasted Boundaries}
\label{sec:proofNFA_C}

This section proves that TMA-MCB (see Definition~\ref{def:nfaContrastedCurve_k}, p.~\pageref{def:nfaContrastedCurve_k}) are theoretically correct.
As usual, being correct means that the following proposition holds.
\begin{proposition}
  The expected number of TMA \meps-meaningful boundaries in a finite set $E$ of random curves is smaller than \meps.
\end{proposition}

\begin{proof}
  For this proof we follow the scheme from Proposition~12 in~\cite{cao08theory}.

  For all $k$, let us denote by $L_k$ the random length of the pieces of $C$ such that $|Du| \geq \mu_k$. From Definition~\ref{def:nfaContrastedCurve_k}, any curve $C$ is \meps-meaningful if there is at least one $0 \leq k < K$ such that
  $N_{ll} \ K \ \widetilde{\bintail} (n \cdot \lsn{2}, L_k; H_c(\mu_k)) < \eps$.
  Let us denote by $E(C, k)$ this event and recall that all probabilities are under $\Hy_0$:
  \begin{equation*}
    \Pr (E(C, k)) \stackrel{\mathrm{def}}{=} \Pr \left( \widetilde{\bintail} (n \cdot \lsn{2}, L_k; H_c(\mu_k) < \frac{\eps}{N_{ll} \ K} \right) \text{.}
  \end{equation*}

  From Lemma~\ref{lem:classic}, we denote
  \begin{align*}
    X & = L_k  & S(x) & = \widetilde{\bintail} (n \cdot \lsn{2}, x; Hc(\mu_k)) \\
    t & = \frac{\eps}{N_{ll} \ K} &\quad \Pr(S(X) <  t) & =  \Pr (E(C, k))
  \end{align*}
  and finally
  \begin{equation*}
    \Pr (E(C, k)) \leq \frac{\eps}{N_{ll} \cdot K} \text{.}
  \end{equation*}

  The event defined by ``C is \meps-meaningful'' is $$E(C) = \bigcup_{0 \leq k < K} E(C, k).$$ Let us denote by $\expectation_{\Hy_0}$ the mathematical expectation under $\Hy_0$. The expected number of \meps-meaningful curves is defined as $\mathbb{E}_{\Hy_0} \left( \sum_{C \in \mathcal{C}} \mathbf{1}_{E(C)} \right)$ where $\mathbf{1}_{A}$ is the indicator function of the set $A$.
  Then
  \begin{equation*}
    \expectation_{\Hy_0} \left( \sum_{C \in \mathcal{C}} \mathbf{1}_{E(C)} \right) \leq \sum_{\substack{ C \in \mathcal{C} \\ 0 \leq k < K }} \Pr \left( E(C, k) \right)
    \leq \sum_{\substack{C \in \mathcal{C} \\ 0 \leq k < K}} \frac{\eps}{N_{ll} \cdot K} = \eps.
  \end{equation*}

\end{proof}
\subsection{Meaningful Contrasted Regular Boundaries}
\label{sec:proofNFA_CR}

TMA \meps-meaningful boundaries (see Definition~\ref{def:nfaContrastedSmoothCurve_k}, p.~\pageref{def:nfaContrastedSmoothCurve_k}) are correct is the following proposition holds.
\begin{proposition}
  The expected number of \meps-meaningful contrasted regular boundaries, obtained with Definition~\ref{def:nfaContrastedSmoothCurve_k}, in a finite random set $E$ of random curves is smaller than \meps.
\end{proposition}

\begin{proof}
  The same assumptions from the previous proof hold.
  
  Let $X_i = \mathbf{1}_{C_i \mathrm{\ is\ meaningful}}$ and $N = \#E$. Let us denote by $\expectation_{\Hy_0}$ the mathematical expectation under $\Hy_0$. Then
  \begin{multline}
    \expectation \left( \sum_{i=1}^{N} \sum_{k=1}^{K_c} \sum_{k'=1}^{K_s} X_i \right) = \\ \expectation \left( \expectation \left( \sum_{i=1}^{n} \sum_{k=1}^{k_c} \sum_{k'=1}^{k_s} X_i \ |\ N = n, K_c = k_c, K_s = k_s \right) \right) \textbf{.}
  \end{multline}
  We have assumed that $N$ is independent from the curves and $K_c$, $K_s$ are input parameters. Thus, conditionally to $N = n$, the law of $\sum_{i=1}^{N} X_i$ is the law of $\sum_{i=1}^{n} Y_i$ where 
    $$ \displaystyle Y_i = \mathbf{1}_{n\, k_c\, k_s\, \max \left(\min_{0 \leq k < k_c}I_c (C_i, k)^2,\  \min_{0 \leq k' < k_s} I_s (C_i, k')^2 \right)  < \eps}. $$
  By the linearity of expectation
  \begin{equation}
    \expectation \left( \sum_{i=1}^{n} \sum_{k=1}^{k_c} \sum_{k'=1}^{k_s} X_i \right)
    = \expectation \left( \sum_{i=1}^{n} \sum_{k=1}^{k_c} \sum_{k'=1}^{k_s} Y_i \right)
    = \sum_{i=1}^{n} \sum_{k=1}^{k_c} \sum_{k'=1}^{k_s} \expectation \left( Y_i \right) \text{.}
  \end{equation}
  Since $Y_i$ is a Bernoulli variable,
  \begin{multline}
    \expectation (Y_i) = \Pr (Y_i = 1)
    = \Pr \left( n\, k_c\, k_s \
    \max \left(
    \begin{split}
    \min_{0 \leq k < k_c}I_c (C_i, k)^2 \\
    \min_{0 \leq k' < k_s} I_s (C_i, k')^2
    \end{split}
    \right)  < \eps \right) =\\
    = \sum_{l=0}^{\infty} \Pr \left( n\, k_c\, k_s
    \max \left(
    \begin{split}
    \min_{0 \leq k < k_c}I_c (C_i, k)^2 \\
    \min_{0 \leq k' < k_s} I_s (C_i, k')^2
    \end{split}
    \right)  < \eps \ \Big|\ L_i=l
    \right)
    \cdot \Pr (L_i=l) \text{.}
  \end{multline}
  Let us finally denote by $\alpha_1 \dots \alpha_l$ the $l$ independent values of $|Du|$ and $\gamma_1 \dots \gamma_{l/s}$ the $l/s$ independent values of $|R_s|$.
  Again, we have assumed that $L_i$ is independent of the gradient and regularity distributions in the image. Thus conditionally to $L_i = l$,
  \begin{multline}
    \Pr \left( n\, k_c\, k_s
    \max \left(
    \begin{split}
    \min_{0 \leq k < k_c} I_c (C_i, k)^2 \\
    \min_{0 \leq k' < k_s} I_s (C_i, k')^2
    \end{split}
    \right) < \eps \ |\ L_i=l \right) = \\
    = \Pr \left( n\, k_c\, k_s
    \max \left(
    \begin{split}
    \min_{0 \leq k < k_c} I_c (C_i, k)^2 \\
    \min_{0 \leq k' < k_s} I_s (C_i, k')^2
    \end{split}
    \right) < \eps \right) = \\
    = \Pr \left(
    \max \left(
    \begin{split}
    \min_{0 \leq k < k_c} I_c (C_i, k) \\
    \min_{0 \leq k' < k_s} I_s (C_i, k')
    \end{split}
    \right) < \left( \frac{\eps}{n\, k_c\, k_s} \right)^{1 / 2} \right) = \\
    = \Pr \left( \min_{0 \leq k < k_c}I_c (C_i, k) < \left( \frac{\eps}{n\, k_c\, k_s} \right)^{1 / 2} \right) \cdot \\
    \Pr \left( \min_{0 \leq k' < k_s}I_s (C_i, k') < \left( \frac{\eps}{n\, k_c\, k_s} \right)^{1 / 2} \right) \text{.}
  \end{multline}
  From proof of Proposition~\ref{prop:nfaContrastedCurve_k},
  \begin{multline}
    \Pr \left( \min_{0 \leq k < k_c}I_c (C_i, k) < \left( \frac{\eps}{n\, k_c\, k_s} \right)^{1 / 2} \right) \cdot \\
    \Pr \left( \min_{0 \leq k' < k_s}I_s (C_i, k') < \left( \frac{\eps}{n\, k_c\, k_s} \right)^{1 / 2} \right) \leq \\
    \leq \left( \frac{\eps}{n\, k_c\, k_s} \right)^{1 / 2} \left( \frac{\eps}{n\, k_c\, k_s} \right)^{1 / 2} = \frac{\eps}{n\, k_c\, k_s} \textbf{.}
  \end{multline}
  Finally
  \begin{equation}
    \expectation (Y_i) \leq \frac{\eps}{n\, k_c\, k_s} \quad \Rightarrow \quad \sum_{i=1}^{n} \sum_{k=1}^{k_c} \sum_{k'=1}^{k_s} \expectation (Y_i) \leq \eps \text{.}
  \end{equation}
  
\end{proof}

\section*{Acknowledgments}
We also thank Rafael Grompone von Gioi for his fruitful comments.
We acknowledge financial support by CNES (R\&T Echantillonnage Irregulier DCT / SI / MO - 2010.001.4673), FREEDOM (ANR07-JCJC-0048-01), Callisto (ANR-09-CORD-003), ECOS Sud U06E01, STIC Amsud (11STIC-01 - MMVPSCV) and the Uruguayan Agency for Research and Innovation (ANII) under grant PR-POS-2008-003.


\begin{thebibliography}{10}
\providecommand{\url}[1]{{#1}}
\providecommand{\urlprefix}{URL }
\expandafter\ifx\csname urlstyle\endcsname\relax
  \providecommand{\doi}[1]{DOI~\discretionary{}{}{}#1}\else
  \providecommand{\doi}{DOI~\discretionary{}{}{}\begingroup
  \urlstyle{rm}\Url}\fi

\bibitem{arnheim}
Arnheim, R.: {Visual Thinking}.
\newblock University of California Press (1969)

\bibitem{astrom95-limitations}
Astrom, K.: {Fundamental limitations on projective invariants of planar
  curves}.
\newblock IEEE Trans. Pattern Anal. Mach. Intell.
  \textbf{17}(1), 77--81 (1995)

\bibitem{attneave54}
Attneave, F.: {Some informational aspects of visual perception.}
\newblock Psychol. Rev. \textbf{61}(3), 183--193 (1954)

\bibitem{canny86}
Canny, J.: {A Computational Approach to Edge Detection}.
\newblock IEEE Trans. Pattern Anal. Mach. Intell.
  \textbf{8}(6), 679--698 (1986)

\bibitem{cao08theory}
Cao, F., Lisani, J.L., Morel, J.M., Mus\'{e}, P., Sur, F.: {A Theory of Shape
  Identification}, \emph{Lecture Notes in Mathematics}, vol. 1948.
\newblock Springer (2008)

\bibitem{cao2005}
Cao, F., Mus\'e, P., Sur, F.: {Extracting Meaningful Curves from Images}.
\newblock J. Math. Imaging Vis. \textbf{22}(2-3), 159--181 (2005)

\bibitem{caselles99}
Caselles, V., Morel, J.M.: {Topographic Maps and Local Contrast Changes in
  Natural Images}.
\newblock Int. J. Comput. Vis. \textbf{33}, 5--27 (1999)

\bibitem{caselles10}
Caselles, V. and Monasse, P.: {Geometric Description of Images as Topographic Maps}.
\newblock Springer (2010)

\bibitem{cavanagh05}
Cavanagh, P.: {The artist as neuroscientist}.
\newblock Nature \textbf{434}(7031), 301--307 (2005)

\bibitem{dmm01}
Desolneux, A., Moisan, L., Morel, J.M.: {Edge Detection by Helmholtz
  Principle}.
\newblock Journal of Mathematical Imaging and Vision \textbf{14}(3), 271--284
  (2001)

\bibitem{desolneux08}
Desolneux, A., Moisan, L., Morel, J.M.: {From Gestalt Theory to Image
  Analysis}, vol.~34.
\newblock Springer-Verlag (2008)

\bibitem{grompone10}
Grompone~von Gioi, R., Jakubowicz, J., Morel, J.M., Randall, G.: {LSD: A Fast
  Line Segment Detector with a False Detection Control}.
\newblock IEEE Trans. Pattern Anal. Mach. Intell.
  \textbf{32}(4), 722--732 (2010)

\bibitem{morelPDEs}
Guichard, F., Morel, J.M., Ryan, R.: {Contrast invariant image analysis and
  PDE's}.
\newblock http://www.cmla.ens-cachan.fr/Membres/morel.html

\bibitem{kanizsa79}
Kanizsa, G.: {Organization in Vision: Essays on Gestalt Perception}.
\newblock Praeger (1979)

\bibitem{kass88}
Kass, M., Witkin, A., Terzopoulos, D.: {Snakes: Active contour models}.
\newblock Int. J. Comput. Vis. \textbf{1}(4), 321--331
  (1988)

\bibitem{krim06}
Krim, H., Yezzi, A. (eds.): {Statistics and Analysis of Shapes}.
\newblock Modeling and Simulation in Science, Engineering and Technology.
  Birkh\"{a}user Boston (2006)

\bibitem{lisani03-shape}
Lisani, J.L., Moisan, L., Morel, J.M., Monasse, P.: {On the theory of planar
  shape}.
\newblock Multiscale Model. Simul. \textbf{1}(1), 1--24 (2003)

\bibitem{matas02-mser}
Matas, J., Chum, O., Martin, U., Pajdla, T.: {Robust wide baseline stereo from
  maximally stable extremal regions}.
\newblock In: BMVC, London, UK (2002)

\bibitem{matheron75}
Matheron, G.: {Random Sets and Integral Geometry}.
\newblock John Wiley \& Sons, NY, USA (1975)

\bibitem{meinhardt08}
Meinhardt, E., Zacur, E., Frangi, A., Caselles, V.: {3D Edge Detection by
  Selection of Level Surface Patches}.
\newblock J. Math. Imaging. Vis.  (2008)

\bibitem{meinhardt08b}
Meinhardt-Llopis, E.: Edge detection by selection of pieces of level lines.
\newblock In: ICIP, San Diego, USA (2008)

\bibitem{moisan98}
Moisan, L.: {Affine plane curve evolution: a fully consistent scheme}.
\newblock IEEE Trans. Image Process. \textbf{7}(3), 411--420 (1998)

\bibitem{monasse00}
Monasse, P., Guichard, F.: {Fast Computation of a Contrast Invariant Image
  Representation}.
\newblock IEEE Trans. Image Process. \textbf{9}(5), 860--872 (2000)

\bibitem{morel09ASIFT}
Morel, J.M., Yu, G.: {ASIFT: A New Framework for Fully Affine Invariant Image
  Comparison}.
\newblock SIAM J. Imaging Sci. \textbf{2}(2), 438--469 (2009)

\bibitem{museThesis}
Mus\'{e}, P.: {On the definition and recognition of planar shapes in digital
  images}.
\newblock Ph.D. thesis, \'{E}cole Normal Sup\'{e}rieure de Cachan (2004)

\bibitem{numericalRecipes}
Press, W., Teukolsky, S., Vetterling, W., Flannery, B.: {Numerical Recipes in
  C}, 2nd edn.
\newblock Cambridge University Press (1992)

\bibitem{rabin09}
Rabin, J., Delon, J., Gousseau, Y.: {A Statistical Approach to the Matching of
  Local Features}.
\newblock SIAM J. Imaging Sci. \textbf{2}(3), 931--958 (2009)

\bibitem{serra83}
Serra, J.: {Image Analysis and Mathematical Morphology}.
\newblock Academic Press, Inc (1983)

\bibitem{tepperPhD}
Tepper, M.: {Detecting clusters and boundaries: a twofold study on shape
  representation}.
\newblock Ph.D. thesis, Universidad de Buenos Aires (2011)

\bibitem{tepper09matching}
Tepper, M., Acevedo, D., Goussies, N., Jacobo, J., Mejail, M.: {A decision step
  for shape context matching}.
\newblock In: ICIP, Cairo, Egypt (2009)

\bibitem{tepper09msc}
Tepper, M., G\'{o}mez, F., Mus\'{e}, P., Almansa, A., Mejail, M.:
  {Morphological Shape Context: Semi-locality and Robust Matching in Shape
  Recognition}.
\newblock In: CIARP, Guadalajara, Mexico (2009)

\bibitem{tepper12ps}
Tepper, M., Mus\'{e}, P., Almansa, A.:
  {Finding Edges by A Contrario Detection of Periodic Subsequences}.
\newblock In: CIARP, Buenos Aires, Argentina (2012)

\bibitem{wertheimer38}
Wertheimer, M.: {Laws of organization in perceptual forms}, pp. 71--88.
\newblock Routledge and Kegan Paul (1938)

\end{thebibliography}

\end{document}